\documentclass[preprint, numbers]{article}
\usepackage[preprint]{neurips_2024}

\usepackage[utf8]{inputenc}
\usepackage[T1]{fontenc}

\usepackage{hyperref}       
\usepackage{url}            
\usepackage{booktabs}       
\usepackage{amsfonts}       
\usepackage{nicefrac}       
\usepackage{microtype}      
\usepackage{xcolor}         

\usepackage{doi}
\usepackage{amsmath, amssymb, amsthm}
\usepackage{mathtools}

\usepackage{subcaption}
\usepackage{cleveref}

\newtheorem{theorem}{Theorem}
\newtheorem{corollary}{Corollary}
\newtheorem{assumption}{Assumption}

\theoremstyle{remark}
\newtheorem{remark}{Remark}

\title{Score Change of Variables}

\author{ 
  Stephen Robbins \\
  Department of Mathematics\\
  University of California, Irvine\\
  \texttt{srobbin2@uci.edu}
}
\date{\today}

\begin{document}

\maketitle

\begin{abstract}
We derive a general change of variables formula for score functions, showing that for a smooth, invertible transformation $\mathbf{y} = \phi(\mathbf{x})$, the transformed score function $\nabla_{\mathbf{y}} \log q(\mathbf{y})$ can be expressed directly in terms of $\nabla_{\mathbf{x}} \log p(\mathbf{x})$. Using this result, we develop two applications: First, we establish a reverse-time Itô lemma for score-based diffusion models, allowing the use of $\nabla_{\mathbf{x}} \log p_t(\mathbf{x})$ to reverse an SDE in the transformed space without directly learning $\nabla_{\mathbf{y}} \log q_t(\mathbf{y})$. This approach enables training diffusion models in one space but sampling in another, effectively decoupling the forward and reverse processes. Second, we introduce generalized sliced score matching, extending traditional sliced score matching from linear projections to arbitrary smooth transformations. This provides greater flexibility in high-dimensional density estimation. We demonstrate these theoretical advances through applications to diffusion on the probability simplex and empirically compare our generalized score matching approach against traditional sliced score matching methods.
\end{abstract}

\section{Introduction}
The change of variables formula is a cornerstone of mathematical analysis, crucial for understanding how quantities transform under coordinate mappings \cite{evans2022partial, folland1999real}. Its applications span numerous fields, from physics to probability theory. In probability theory, this formula describes how probability density functions (PDFs) relate under a differentiable invertible transformation $\phi$: if a random variable $\mathbf{X}$ follows a density $p(\mathbf{x})$, then $\mathbf{Y} = \phi(\mathbf{X})$ follows a density
\begin{equation}
q(\mathbf{y}) = p(\phi^{-1}(\mathbf{y})) \left| \det \mathbf{J}_{\phi^{-1}}(\mathbf{y}) \right|,
\end{equation}
where $\mathbf{J}_{\phi^{-1}}$ denotes the Jacobian matrix of the inverse transformation $\phi^{-1}$.

Recently, the score function $\nabla_{\mathbf{x}} \log p(\mathbf{x})$ has gained significant attention in machine learning, playing a pivotal role in diffusion models for generative modeling, statistical inference, and density estimation \cite{song2020score, ho2020denoising, vincent2011connection, efron2011tweedie}. State-of-the-art diffusion models, such as those in \cite{dhariwal2021diffusion}, have demonstrated remarkable success in generating high-fidelity samples by leveraging the score function. Similarly, score matching techniques \cite{hyvarinen2005estimation, vincent2011connection, song2019slicedscorematchingscalable} offer an efficient way to estimate the score directly from data, bypassing the need for explicit density estimation and enabling scalability to high-dimensional problems.

Despite the increasing prominence of the score function, its transformation properties under coordinate changes remain widely unexplored. In this work, we address this limitation by deriving a change of variables formula for score functions. Our formula expresses the score function in the transformed space, $\nabla_{\mathbf{y}} \log q(\mathbf{y})$, directly in terms of the score function in the original space, $\nabla_{\mathbf{x}} \log p(\mathbf{x})$, and derivatives of the transformation $\phi$. This result provides a  mathematical tool for adapting score-based methods to transformed spaces, analogous to how the classical change of variables formula facilitates the transformation of probability densities.

Furthermore, we explore the implications of our score function transformation formula in the context of diffusion models and score matching. We derive a reverse-time Itô lemma that enables efficient sampling in transformed spaces using score functions learned in the original space, effectively decoupling the forward and reverse processes. This decoupling allows for greater flexibility in designing diffusion models tailored to specific geometries or constraints. Additionally,  we introduce Generalized Sliced Score Matching (GSSM), which extends the idea of linear projections in sliced score matching. Similar to how sliced score matching reduces dimensionality using a linear projection, GSSM achieves this by projecting onto a one-dimensional space defined by the gradient of a general smooth function, providing a non-linear generalization of the original concept. This generalization expands the applicability of sliced score matching, allowing for more flexible and potentially more accurate score estimation in diverse settings. 

Our main contributions are:
\begin{enumerate}
\item \textbf{Score Function Change of Variables:}  We establish a general and explicit formula for transforming score functions under smooth invertible mappings.

\item \textbf{Reverse-Time Itô Lemma:} We develop a reverse-time Itô lemma that allows sampling in transformed spaces by leveraging score functions learned in the original space, enabling more flexible diffusion models.

\item \textbf{Generalized Sliced Score Matching:} We extend sliced score matching \cite{song2019slicedscorematchingscalable} by generalizing the linear projection in the original method to the gradient of an arbitrary smooth function from $\mathbb{R}^n \text{ to } \mathbb{R}$.
\end{enumerate}

By providing these theoretical contributions, our work paves the way for a deeper understanding and broader application of score-based methods in various machine learning domains. The ability to seamlessly transform score functions across coordinate systems opens up new possibilities for developing more sophisticated and effective generative models and inference algorithms.

\section{Main Result}

We present a change of variables formula for score functions under smooth, bijective mappings, beginning with the necessary assumptions.

\subsection{Assumptions}
\begin{assumption}[Bijectivity and Smoothness in $\mathbb{R}^n$]
\label{assumption:smooth_bijectivity_support_multidim}
Let \( (\Omega, \mathcal{F}, \mathbb{P}) \) be a probability space, and let \( \mathbf{X}: \Omega \to \mathbb{R}^n \) be a random vector with probability density function \( p \). Assume that \( \nabla_{\mathbf{x}} \log p(\mathbf{x}) \) is well-defined on \( \mathbb{R}^n \) and that \( \phi: U \to U_1 \), with \(U, U_1 \subset \mathbb{R}^n\), is a bijective and twice continuously differentiable mapping. Let \( q(\mathbf{y}) \) denote the probability density function of the transformed variable \( \mathbf{Y} = \phi(\mathbf{X}) \).
\end{assumption}

\subsection{Score Change of Variables}

\begin{theorem}[Score Change of Variables: \( \mathbb{R}^n \to \mathbb{R}^n \)]
\label{thm:multi_dimensional}
Suppose Assumption \ref{assumption:smooth_bijectivity_support_multidim} holds. Then, the score function \( \nabla_{\mathbf{y}} \log q(\mathbf{y}) \) can be expressed as:
\begin{align*}
\nabla_{\mathbf{y}} \log q(\mathbf{y}) =  \mathbf{J}_{\phi^{-1}}(\mathbf{y})^\top \nabla_{\mathbf{x}} \log p(\mathbf{x})  +\nabla_{\mathbf{x}}\cdot \left(\mathbf{J}_{\phi^{-1}}(\phi(\mathbf{x}))^\top\right)\bigg|_{\mathbf{x}=\phi^{-1}(\mathbf{y})}.
\end{align*}
\end{theorem}

\begin{remark}
Throughout this paper:
\begin{itemize}
\item $\mathbf{J}_\phi$ denotes the Jacobian matrix of the transformation $\phi$.
\item For any matrix-valued function $\mathbf{A}(\mathbf{x}): \mathbb{R}^n \to \mathbb{R}^{n \times n}$, we define its divergence row-wise:
\[
\nabla_{\mathbf{x}}\cdot \mathbf{A}(\mathbf{x}) := \left[\nabla_{\mathbf{x}}\cdot \mathbf{A}_i(\mathbf{x})\right]_{i=1}^n \in \mathbb{R}^n,
\]
where $\mathbf{A}_i$ is the $i$-th row of $\mathbf{A}$.
\item We will frequently work with transformations and their inverses. When a transformation \(\mathbf{y} = \phi(\mathbf{x})\) is defined, we may drop the explicit evaluation notation  \(\big|_{\mathbf{x}=\phi^{-1}(\mathbf{y})}\), and it will be implicitly understood that \(\mathbf{x}\) is substituted with \(\phi^{-1}(\mathbf{y})\) within subsequent formulas, unless otherwise specified.
\end{itemize}
\end{remark}

\begin{proof}
Starting from the change of variables formula for densities, \(q(\mathbf{y}) = p(\phi^{-1}(\mathbf{y})) |\det(\mathbf{J}_{\phi^{-1}}(\mathbf{y}))|\), we take the gradient of the log-density and apply the chain rule and the logarithmic derivative of the determinant. Detailed calculations are in \cref{appendix:proofs}.
\end{proof}

\subsection{Special Cases}

The multidimensional result also covers important special cases, presented as corollaries.

\begin{corollary}[Score Change of Variables: \( \mathbb{R} \to \mathbb{R} \)]
\label{cor:score_change_of_variables_1d}
Let \( X: \Omega \to \mathbb{R} \) be a random variable with probability density function \( p \). Assume that \( \nabla_x \log p(x) \) is well-defined on \( \mathbb{R} \) and that \( \phi: U \to U_1 \), with \(U, U_1 \subset \mathbb{R}\), is a bijective and twice continuously differentiable mapping. Let \( q(y) \) denote the probability density function of the transformed variable \( Y = \phi(X) \). Then:
\[
\nabla_y \log q(y) = (\phi^{-1})'(y) \nabla_x \log p(\phi^{-1}(y)) + \frac{(\phi^{-1})''(y)}{(\phi^{-1})'(y)},
\]
or, equivalently,
\[
\nabla_{y} \log q(y) = \frac{1}{(\phi'(x))^2}\left(\phi'(x) \nabla_{x} \log p(x) - \phi''(x)\right).
\]
\end{corollary}

\begin{corollary}[Score Change of Variables: \( \mathbb{R}^n \to \mathbb{R} \)]
\label{cor:score_change_variables_different_dimensions}
Let \( (\Omega, \mathcal{F}, \mathbb{P}) \) be a probability space, and let \( \mathbf{X}: \Omega \to \mathbb{R}^n \) have probability density function \( p \). Let \( v: U \to U_1 \), where \(U \subset \mathbb{R}^n\) and \(U_1 \subset \mathbb{R}\), be a twice continuously differentiable mapping. Assume that \(\|\nabla v(\mathbf{x})\|_2 \neq 0\) for all \(\mathbf{x} \in U\). Let \( q(y) \) be the density of \( Y = v(\mathbf{X}) \in \mathbb{R} \). Then:
\begin{align*}
\nabla_y \log q(y) = {}& \frac{1}{\|\nabla v(\mathbf{x})\|_2^2} \Big( 
    \nabla v(\mathbf{x})^\top \nabla_{\mathbf{x}} \log p(\mathbf{x}) - \Delta v(\mathbf{x})  - \sum_{i=1}^n \left( \frac{\partial v}{\partial x_i} \right)^2 
    \frac{\partial}{\partial y} \log p((x_j)_{j \neq i} \mid y) \Big).
\end{align*}
\end{corollary}

\begin{proof}
The proof constructs a sequence of transformations involving $v$, applies Theorem \ref{thm:multi_dimensional} to each step, and then combines the results. See Appendix~\ref{appendix:proofs} for the full derivation.
\end{proof}

\section{Applications}
We present two applications of the score change of variables: a reverse-time Itô lemma for transformed diffusion processes and a generalized sliced score matching method (Section~\ref{sec:transformed_score_matching}). Both applications leverage the score change of variables formula for more flexible training and sampling.

\subsection{Reverse-Time Itô Lemma}
\label{sec:reverse_time_ito}
Score-based diffusion models \cite{dhariwal2021diffusion,ho2022cascaded,rombach2022high} generate samples by reversing a forward SDE that progressively adds noise to the data. This reversal relies on the score function of the marginal distributions, $\nabla_{\mathbf{x}} \log p_t(\mathbf{x})$. Given a forward SDE:
\begin{equation}
\label{eq:forward_sde}
\mathrm{d}\mathbf{X} = \mathbf{f}(\mathbf{X}, t) \, \mathrm{d}t + \mathbf{G}(\mathbf{X}, t) \, \mathrm{d}\mathbf{W}(t),
\end{equation}
where $\mathbf{W}(t)$ is a standard Wiener process, we define the reverse drift:
\begin{align}
\label{eq:reverse_drift}
\bar{\mathbf{f}}(\mathbf{x}, t) = {}& \mathbf{f}(\mathbf{x}, t) - \mathbf{G}(\mathbf{x},t)\mathbf{G}(\mathbf{x},t)^\top \nabla_{\mathbf{x}} \log p_t(\mathbf{x}) - \nabla_{\mathbf{x}} \cdot \left(\mathbf{G}(\mathbf{x},t)\mathbf{G}(\mathbf{x},t)^\top\right).
\end{align}

The corresponding reverse-time SDE \cite{anderson1982reverse, song2020score} evolves backward in time from $t = T$ to $t = 0$. Let $\bar{\mathbf{X}}_t$ denote the reverse process state at time $t$, where $t$ decreases from $T$ to $0$. The differential $\mathrm{d}t$ now represents an infinitesimal negative time increment:
\begin{align}
\label{eq:reverse_sde_general}
\mathrm{d}\bar{\mathbf{X}} &= \bar{\mathbf{f}}(\bar{\mathbf{X}}, t) \, \mathrm{d}t + \mathbf{G}(\bar{\mathbf{X}}, t) \, \mathrm{d}\bar{\mathbf{W}}(t),
\end{align}
where $\bar{\mathbf{W}}(t)$ is a reverse-time Wiener process (adapted to a decreasing filtration). 

In many scenarios, we prefer to perform this generative diffusion process in a transformed space or with a more complex forward process, for instance when data lies on a manifold \cite{lou2023reflected, de2022riemannian}. Let $\phi: \mathbb{R}^n \times [0, T] \to \mathbb{R}^n$ be a bijective $C^{2,1}$ transformation. The transformed process $\mathbf{Y}(t) = \phi(\mathbf{X}(t), t)$ for $t\in [0, T]$ satisfies a new SDE \cite{karatzas2014brownian}:
\begin{equation}
\label{eq:forward_sde_transformed}
\mathrm{d}\mathbf{Y} = \tilde{\mathbf{f}}(\mathbf{Y}, t) \, \mathrm{d}t + \tilde{\mathbf{G}}(\mathbf{Y}, t) \, \mathrm{d}\mathbf{W}(t),
\end{equation}
where the transformed drift $\tilde{\mathbf{f}}$ and diffusion $\bar{\mathbf{G}}$ are obtained via Itô's Lemma. For $\mathbf{x} = \phi^{-1}(\mathbf{y}, t)$:
\begin{align}
\tilde{\mathbf{f}}(\mathbf{y}, t) = {}& \frac{\partial \phi(\mathbf{x}, t)}{\partial t} + \mathbf{J}_{\phi}(\mathbf{x}, t) \mathbf{f}(\mathbf{x}, t) + \frac{1}{2} \text{Tr}\left[ \mathbf{G}(\mathbf{x}, t)^\top \mathbf{H}_{\phi}(\mathbf{x}, t) \mathbf{G}(\mathbf{x}, t) \right], \label{eq:transformed_drift} \\
\tilde{\mathbf{G}}(\mathbf{y}, t) = {}& \mathbf{J}_{\phi}(\mathbf{x}, t) \mathbf{G}(\mathbf{x}, t), \label{eq:transformed_diffusion}
\end{align}
where $\mathbf{H}_{\phi}$ is the Hessian tensor of $\phi$ with respect to $\mathbf{x}$.

Applying our score change of variables formula leads to the main result for transformed diffusion processes:

\begin{corollary}[Reverse-Time Itô Lemma]
\label{thm:Reverse_Time_Ito_Lemma}
Let $\phi: \mathbb{R}^n \times [0, T] \to \mathbb{R}^n$ be a bijective $C^{2,1}$ transformation. Let $p_t(\mathbf{x})$ and $q_t(\mathbf{y})$ be the probability density functions at time $t$ of $\mathbf{X}$ and $\mathbf{Y} = \phi(\mathbf{X}, t)$, respectively. The reverse-time SDE of $\mathbf{Y}$ is:
\begin{align}
\label{eq:reverse_sde_Y_final}
\mathrm{d}\bar{\mathbf{Y}} =  \hat{\mathbf{f}}(\bar{\mathbf{Y}}, t)dt  + \tilde{\mathbf{G}}(\mathbf{\bar{Y}}, t)\mathrm{d}\bar{\mathbf{W}}(t).
\end{align}
where the diffusion $\tilde{\mathbf{G}}$ is defined in \eqref{eq:transformed_diffusion}  and the reverse drift $\hat{\mathbf{f}}$ is given by:
\begin{align}
\hat{\mathbf{f}}(\mathbf{y}, t) &=  \frac{\partial \phi(\mathbf{x}, t)}{\partial t} + \mathbf{J}_{\phi}(\mathbf{x}, t) \bar{\mathbf{f}}(\mathbf{x}, t) - \frac{1}{2} \text{Tr}\left[ \mathbf{G}(\mathbf{x}, t)^\top \mathbf{H}_{\phi}(\mathbf{x}, t) \mathbf{G}(\mathbf{x}, t) \right].\label{eq:transformed_reverse_drift}
\end{align}
\end{corollary}

This result can be interpreted through two complementary perspectives:
\begin{enumerate}
    \item \textbf{Reversing the Transformed Process:} Directly applying the time-reversal formula \eqref{eq:reverse_drift} to the forward SDE of $\mathbf{Y}$ in \eqref{eq:forward_sde_transformed}, substituting $\tilde{\mathbf{f}}$ and $\tilde{\mathbf{G}}$ into \eqref{eq:reverse_drift}.
    
    \item \textbf{Transforming the Reversed Process:} Applying Itô's Lemma to the original reverse process $\bar{\mathbf{X}}_t$ \eqref{eq:reverse_sde_general} under $\phi(\cdot,t)$. 
\end{enumerate}

We show  in Appendix~\ref{appendix:reverse_time_ito_lemma} that both approaches simplify to \eqref{eq:reverse_sde_Y_final} for $\bar{\mathbf{Y}}_t$, confirming the path-wise equivalence:
\begin{equation*}
    \text{Reverse}(\phi(\mathbf{X}_t, t)) \equiv \phi(\text{Reverse}(\mathbf{X}_t), t).
\end{equation*}

In practice, we approximate $\nabla_{\mathbf{x}} \log p_t(\mathbf{x})$ with a neural network $s_\theta(\mathbf{x}, t)$ trained via denoising score matching \cite{vincent2011connection}. To account for the transformation, we modify the training loss to incorporate appropriate Jacobian weights (see Appendix~\ref{appendix:training} for details):
\begin{align} 
    L(\theta) &=  \mathbb{E}_{t} \Big[ \lambda(t) \, \mathbb{E}_{\mathbf{X}(0)} \, \mathbb{E}_{\mathbf{X} | \mathbf{X}(0)} \Big\| \mathbf{J}_{\phi^{-1}}(\phi(\mathbf{X},t), t)^{\top}  \Big( s_\theta(\mathbf{X}, t) - \nabla_{\mathbf{x}} \log p_{0t}(\mathbf{X} | \mathbf{X}(0)) \Big) \Big\|^2 \Big]
    \label{eq:weighted_loss_multidim}
\end{align}
This weighted objective ensures that errors in the score approximation transform appropriately under $\phi$.

\subsection{Transformed Score Matching}
\label{sec:transformed_score_matching}
While the previous section applied Theorem~\ref{thm:multi_dimensional} to continuous-time processes, we now generalize sliced score matching using Corollary~\ref{cor:score_change_variables_different_dimensions}. This \emph{Generalized Sliced Score Matching} (GSSM) framework offers greater flexibility in data projection while retaining computational efficiency. We first briefly review score matching and sliced score matching, and then introduce GSSM.

\subsubsection{Background: Score Matching}

Score matching \cite{hyvarinen2005estimation} estimates the score function $\nabla_{\mathbf{x}} \log p_d(\mathbf{x})$ of a probability density $p_d(\mathbf{x})$ from samples $\mathbf{x}\sim p_d$, without requiring an explicit density form. Consider the loss for a score model $s_\theta(\mathbf{x})$:
\begin{equation}
\label{eq:score_matching_loss}
\mathcal{L}_{\text{SM}}(s_\theta) = \frac{1}{2} \mathbb{E}_{p_d} \left[ \| s_\theta(\mathbf{X}) - \nabla_{\mathbf{x}} \log p_d(\mathbf{X}) \|^2 \right].
\end{equation}

Since $\nabla_{\mathbf{x}} \log p_d(\mathbf{x})$ is unknown, we cannot compute this directly. Using integration by parts and suitable boundary conditions, the loss can be rewritten without directly involving $\nabla_{\mathbf{x}}\log p_d(\mathbf{x})$:
\begin{equation}
\label{eq:score_matching_loss_expanded}
\mathcal{L}_{\text{SM}}(s_\theta) = \frac{1}{2} \mathbb{E}_{p_d} \left[ \| s_\theta(\mathbf{X}) \|^2 \right] + \mathbb{E}_{ p_d} \left[ \nabla_{\mathbf{x}} \cdot s_\theta(\mathbf{X}) \right].
\end{equation}

\subsubsection{Sliced Score Matching} 
In high-dimensional settings, evaluating $\mathbb{E}_{p_d}[\nabla_{\mathbf{x}} \cdot s_\theta(\mathbf{X})]$ can be prohibitively expensive. Sliced score matching \cite{song2019slicedscorematchingscalable} addresses this by projecting onto random one-dimensional subspaces. Specifically, it uses random vectors \(\mathbf{v}\in\mathbb{R}^n\) drawn from a distribution \(p_\mathbf{v}\), where \(\mathbb{E}_{p_\mathbf{v}}[\|\mathbf{v}\|^2]<\infty\) and \(\mathbb{E}_{p_\mathbf{v}}[\mathbf{v}\mathbf{v}^\top]\succ 0\). Under these conditions, the loss becomes:
\begin{align}
\label{eq:sliced_score_matching_loss}
\mathcal{L}_{\text{SSM}}(s_\theta) = & \frac{1}{2}  \mathbb{E}_{p_d} \mathbb{E}_{ p_\mathbf{v}} \left[ \left( \mathbf{v}^\top s_\theta(\mathbf{X}) \right)^2 \right] + \mathbb{E}_{p_d} \mathbb{E}_{ p_\mathbf{v}}\left[\mathbf{v}^\top \nabla_{\mathbf{x}} \left( \mathbf{v}^\top s_\theta(\mathbf{X}) \right) \right].
\end{align}
This objective is equivalent to the original score matching loss up to a constant, yet is more tractable computationally in high dimensions.

\subsubsection{Generalized Sliced Score Matching}

We now introduce \emph{Generalized Sliced Score Matching} (GSSM), which extends sliced score matching to use arbitrary smooth transformations. Let \( v: \mathbb{R}^n \to \mathbb{R} \) be a twice continuously differentiable random function drawn from an independent distribution \( p_v \). Using Corollary~\ref{cor:score_change_variables_different_dimensions}, we derive:
\begin{align}
\label{eq:gssm_loss}
\mathcal{L}_{\text{GSSM}}(s_\theta) &= \frac{1}{2}\mathbb{E}_{p_d} \mathbb{E}_{p_v} \left[ \left( \nabla_{\mathbf{x}}v(\mathbf{X})^\top s_\theta(\mathbf{X}) \right)^2 \right]  + \mathbb{E}_{p_d} \mathbb{E}_{p_v}  \left[(\nabla_{\mathbf{x}} v(\mathbf{X}))^\top \nabla_{\mathbf{x}} 
\bigl(s_\theta(\mathbf{X})\nabla_{\mathbf{x}} v(\mathbf{X})\bigr) \right] \nonumber \\
&\quad + \mathbb{E}_{p_d} \mathbb{E}_{p_v}  \left[ s_\theta(\mathbf{X})^\top \mathbf{H}_v(\mathbf{X})\nabla_{\mathbf{x}}v(\mathbf{X}) \right]  \mathbb{E}_{p_d} \mathbb{E}_{p_v}  \left[ \nabla_{\mathbf{x}}v(\mathbf{X})^\top s_\theta(\mathbf{X})  \Delta v(\mathbf{X})  \right].
\end{align}

The extra terms compared to SSM account for the nonlinearity of $v$ via its Hessian $\mathbf{H}_v$ and Laplacian $\Delta v$.

When $v(\mathbf{x}) = \mathbf{v}^\top \mathbf{x}$ is linear, we have $\nabla_{\mathbf{x}} v(\mathbf{x}) = \mathbf{v}$, $\mathbf{H}_v = 0$, and $\Delta v = 0$, reducing GSSM to the original sliced score matching loss \eqref{eq:sliced_score_matching_loss}. Thus, GSSM strictly generalizes sliced score matching, providing more flexibility and potentially improved performance in complex, high-dimensional scenarios.

\subsubsection{Derivation Outline}
To derive GSSM, we first apply standard score matching in the transformed space $y = v(\mathbf{x})$. Then, using Corollary~\ref{cor:score_change_variables_different_dimensions}, we substitute the transformed score back into the original space. Integration by parts under appropriate conditions (see Appendix~\ref{appendix:gssm_derivation}) yields Equation~\ref{eq:gssm_loss}. The complete derivation is in Appendix~\ref{appendix:gssm_derivation}.

\section{Examples}
\subsection{Diffusion on the Probability Simplex for Chess Positions}
\label{sec:chess_diffusion}

We demonstrate our reverse-time Itô lemma by generating chess positions through diffusion on a probability simplex. Prior work has explored the same diffusion process on the simplex by training directly in a constrained space \cite{floto2023diffusion}, but our approach enables training in an unconstrained space with a Gaussian score model while still sampling in the simplex. 

\subsubsection{Representing Chess Positions}

A chess position can be represented as a point in a 13-dimensional probability simplex. We focus on the projected simplex defined as:
\[
 \{ \mathbf{y} \in [0,1]^{12} \mid \sum_{i=1}^{12} y_i \leq 1 \}.
\]
Here, each element $\mathbf{y}$ specifies probabilities for 12 possible piece types (6 for White and 6 for Black) that could occupy a single square on the chessboard. The probability of the square being empty is implicitly given by \(1 - \sum_{i=1}^{12} y_i\). Extending this construction to all 64 squares, we represent a full chess position as a collection of such probability vectors across the board.

\subsubsection{Transformation Between Spaces}

We employ the additive logistic transformation \cite{atchison1980logistic} and its inverse to map between $\mathbb{R}^{12}$ and the projected simplex. For $i = 1, \dots, 12$:
\begin{align}
\phi_i(\mathbf{x}) &= \frac{e^{x_i}}{1 + \sum_{j=1}^{12} e^{x_j}}, \\
\phi^{-1}_i(\mathbf{y}) &= \log \left( \frac{y_i}{1 - \sum_{j=1}^{12} y_j} \right).
\end{align}
This mapping allows us to freely train a score-based model in an unconstrained Euclidean space, then transform and sample in the simplex domain where the probabilities must sum to at most one.

\subsubsection{Training and Sampling}

We train a score-based model $s_\theta(\mathbf{x}, t)$ in $\mathbb{R}^{12}$. Chess positions are represented with a softened one-hot encoding to avoid infinite values when applying the inverse transformation. The score model is trained using a Variance Preserving (VP) SDE \cite{song2020score}:
\begin{equation}
\mathrm{d}\mathbf{X} = -\frac{1}{2} \beta(t) \mathbf{X} \, \mathrm{d}t + \sqrt{\beta(t)} \, \mathrm{d}\mathbf{W},
\end{equation}
where \(\beta(t)\) is the noise schedule. During inference, we apply the reverse-time Itô lemma to perform diffusion directly in the projected simplex. By leveraging the learned scores in the unconstrained space and the known transformation $\phi$, we obtain the appropriate transformed SDE coefficients (derived in Appendix~\ref{appendix:chess_derivations}) and sample valid chess positions on the simplex.

\subsubsection{Controlling Piece Density}

The transformed reverse drift term, $\hat{\mathbf{f}}(\mathbf{Y}, t)$, interacts with the simplex geometry to influence piece distributions. To gain intuitive control, we introduce a scaling factor $w$:
\begin{equation}
\hat{\mathbf{f}}_w(\mathbf{Y}, t) = w \hat{\mathbf{f}}(\mathbf{Y}, t).
\end{equation}

Decreasing $w$ biases the distribution toward fewer pieces (increasing empty squares), while increasing it encourages more pieces. This  scalar parameter provides interpretable geometric control over the final configurations, allowing practitioners to shape the sampling process easily. The generated positions in Figure~\ref{fig:chess_samples} demonstrate how varying $w$ changes piece density without losing consistency in the underlying chess position representation.

\begin{figure*}[tb]  
    \centering
    \begin{minipage}[b]{0.23\textwidth}
        \centering
        \includegraphics[width=\textwidth]{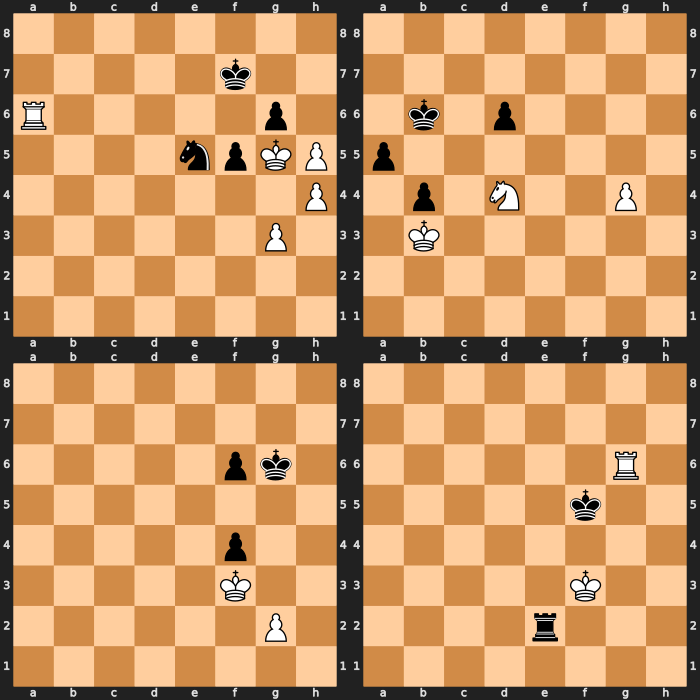}
        \caption{ \( w = 0.8 \).}
    \end{minipage}
    \hfill
    \begin{minipage}[b]{0.23\textwidth}
        \centering
        \includegraphics[width=\textwidth]{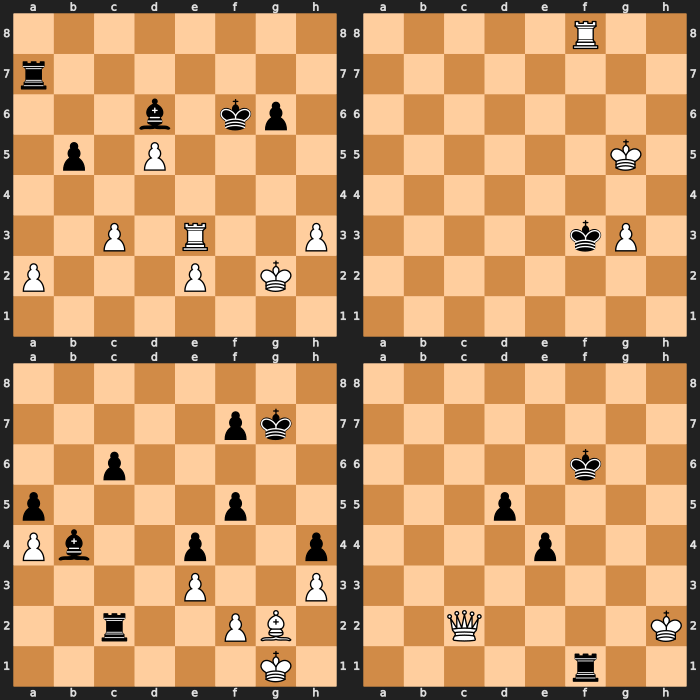}
        \caption{ \( w = 0.9 \).}
    \end{minipage}
    \hfill
    \begin{minipage}[b]{0.23\textwidth}
        \centering
        \includegraphics[width=\textwidth]{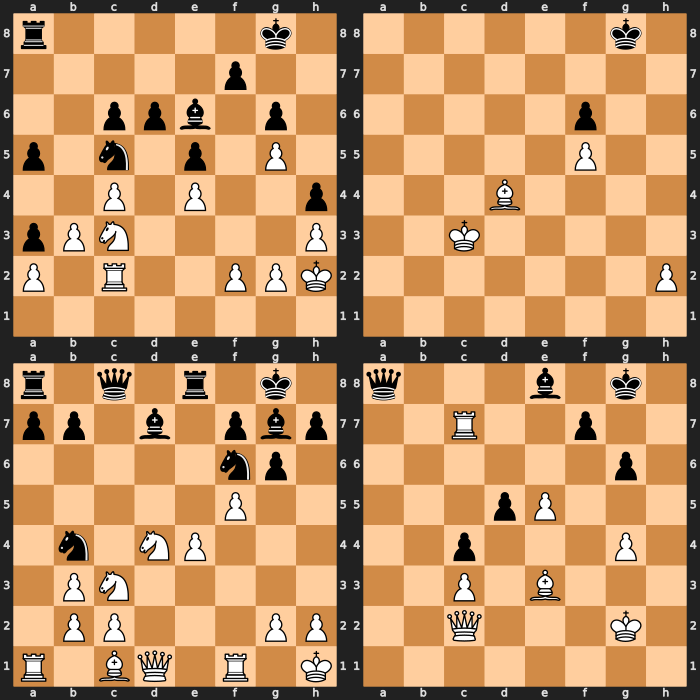}
        \caption{ \( w = 1.0\).}
    \end{minipage}
    \hfill
    \begin{minipage}[b]{0.23\textwidth}
        \centering
        \includegraphics[width=\textwidth]{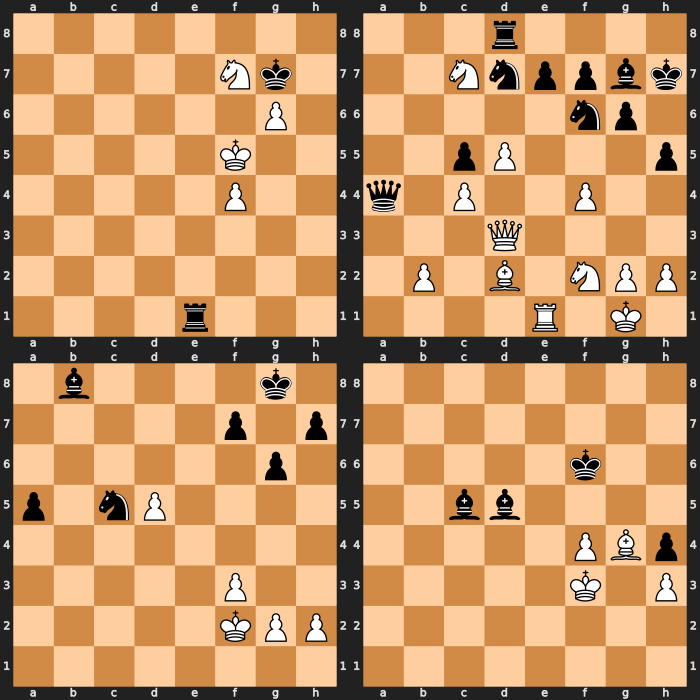}
        \caption{ \( w = 1.1\).}
    \end{minipage}
    \caption{Generated chess positions with different values of \( w \). Increasing \( w \) increases the number of pieces on the board.}
    \label{fig:chess_samples}
\end{figure*}

\subsection{Density Estimation with Generalized Sliced Score Matching}
\label{sec:gssm_experiment}

We evaluate Generalized Sliced Score Matching (GSSM) on density estimation tasks using deep kernel exponential families (DKEFs), comparing its performance against traditional score matching (SM) and sliced score matching (SSM) \cite{song2019slicedscorematchingscalable}. We also introduce and evaluate a variance-reduced version of GSSM (GSSM-VR).

\subsubsection{Choice of Transformation}

For GSSM, we use a quadratic form as our nonlinear transformation:
\begin{equation}
\label{eq:v_quadratic}
v(\mathbf{x}) = \frac{1}{2}\mathbf{x}^\top \mathbf{A} \mathbf{x} + \mathbf{b}^\top \mathbf{x},
\end{equation}
where $\mathbf{A}$ is a random symmetric matrix and $\mathbf{b}$ is a random vector. The entries of $\mathbf{A}$ and $\mathbf{b}$ are drawn from zero-mean distributions with specified variances: $\sigma_1^2$ for diagonal entries of $\mathbf{A}$, $\sigma_2^2$ for off-diagonal entries, and $\sigma_3^2$ for entries of $\mathbf{b}$. The gradient of this transformation, which determines the direction of projection in GSSM, is:
\begin{equation}
\label{eq:v_gradient}
\nabla_{\mathbf{x}} v(\mathbf{x}) = \mathbf{A}\mathbf{x} + \mathbf{b}.
\end{equation}

\subsubsection{Variance Reduction}

Depending on the chosen transformation, we may reduce the variance of our estimator by analytically integrating out some randomness. For linear transformations (as in standard sliced score matching), integrating out the random directions yields the known variance-reduced SSM objective \cite{song2019slicedscorematchingscalable}:
\begin{equation}
\label{eq:ssm_vr}
\mathcal{L}_{\text{SSM-VR}}(s_\theta) = \mathbb{E}_{p_d} \left[ \frac{1}{2} \|s_\theta(\mathbf{X})\|^2 + \mathbb{E}_{p_\mathbf{v} } \left[ \mathbf{v}^\top \nabla_{\mathbf{x}} s_\theta(\mathbf{X}) \mathbf{v} \right] \right].
\end{equation}

Similarly, for the quadratic transformation, integrating out some of the randomness from $\mathbf{A}$ and $\mathbf{b}$ leads to a variance-reduced GSSM (GSSM-VR) objective:
\begin{equation}
    \label{eq:gssm_vr}
    \begin{aligned}
    \mathcal{L}_{\text{GSSM-VR}}(s_\theta) = \mathbb{E}_{p_d} \Big[
    &\frac{1}{2} L_1 + L_2 + \mathbb{E}_{\mathbf{A}, \mathbf{b}} \left[ L_3 \right] \Big],
    \end{aligned}
\end{equation}
where
\begin{align*}
L_1 &= (\sigma_1^2 - 2\sigma_2^2) \sum_{i} s_{\theta i}(\mathbf{X})^2 x_i^2  + \sigma_2^2 (\| \mathbf{X} \|^2 \| s_\theta(\mathbf{X}) \|^2 + (s_\theta(\mathbf{X})^\top \mathbf{X})^2) + \sigma_3^2 \| s_\theta(\mathbf{X}) \|^2 \\
L_2 &= (2\sigma_1^2 + (n - 1) \sigma_2^2) s_\theta(\mathbf{X})^\top \mathbf{X} \\
L_3 &= (\mathbf{A}\mathbf{X} + \mathbf{b})^\top \nabla_{\mathbf{x}} s_\theta(\mathbf{X}) (\mathbf{A}\mathbf{X} + \mathbf{b})
\end{align*}

The full derivation for the quadratic case is provided in Appendix~\ref{appendix:quadratic_gssm}.

\subsubsection{Score Function Representation in DKEF Models}

We follow \cite{song2019slicedscorematchingscalable, wenliang2021learningdeepkernelsexponential} and use DKEF models to evaluate different score matching methods. A DKEF approximates an unnormalized log density as:
\begin{equation}
\log \tilde{p}_f(\mathbf{x}) = f(\mathbf{x}) + \log q_0(\mathbf{x}), \quad
f(\mathbf{x}) = \sum_{l=1}^L \alpha_l k(\mathbf{x}, \mathbf{z}_l),
\end{equation}
where $k$ is a mixture of Gaussian kernels centered at learned inducing points $\mathbf{z}_l$. The kernel features are defined through a neural network, and we jointly optimize the network parameters, kernel weights $\alpha_l$, and inducing points $\mathbf{z}_l$.

Additional details on the DKEF model architecture are provided in Appendix \ref{appendix:gssm_dkef}.

\subsubsection{Experimental Setup}

We evaluate SM, SSM, GSSM, and their variance-reduced variants on three UCI datasets (Parkinsons, RedWine, WhiteWine) \cite{asuncion2007uci}, following the training protocol and hyperparameters from \cite{song2019slicedscorematchingscalable}. For the quadratic transformation, we set $\sigma_3^2 = 1$ and choose the matrix variances as $\sigma_2^2 = 1/\sqrt{n}$ and $\sigma_1^2 = 2/\sqrt{n}$, where $n$ is the input dimension. The vector $\mathbf{b}$ is sampled from a Rademacher distribution, and entries of $\mathbf{A}$ are sampled from a Gaussian distribution with the given variances. This setup allows efficient sampling of $\mathbf{A}\mathbf{x}$ without explicitly constructing $\mathbf{A}$ (see Appendix~\ref{appendix:quadratic_gssm}). For SSM-VR and SSM, we sample the random projection vectors $\mathbf{v}$ from a Rademacher distribution, following \cite{song2019slicedscorematchingscalable}.

\subsubsection{Results}

\begin{figure*}[t]
    \begin{center}
    \includegraphics[width=0.9\linewidth]{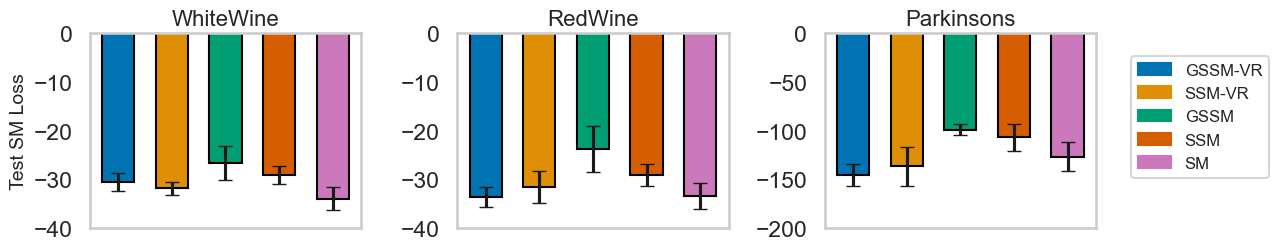}
    \vspace{-5mm}
    \end{center}
    \caption{Exact Score Matching loss across three UCI datasets; lower loss is better.}
    \label{fig:dkef2}
\end{figure*}

Figure~\ref{fig:dkef2} shows the score matching loss for different methods across the three datasets. Our variance-reduced GSSM (GSSM-VR) outperforms SSM-VR on RedWine and Parkinsons, while remaining competitive on WhiteWine. GSSM-VR also demonstrates lower variance in its estimates for RedWine and Parkinsons, indicating more stable optimization.

The non-variance-reduced GSSM performs less favorably than SSM under identical conditions, likely due to higher estimator variance impacting the optimization process. As in \cite{song2019slicedscorematchingscalable}, training was halted after 200 steps of no validation improvement, which proved limiting for GSSM due to its higher variance. Nevertheless, the key contribution of GSSM is its flexibility: extending score matching beyond linear projections allows tailoring the method to data distributions more effectively. Although the quadratic transformations used here are just one example, the results suggest further exploration into choosing more suitable transformation families for specific datasets or applications.

\section{Discussion}

This work explores a change of variables formula for score functions, providing a principled approach to adapt score-based methods across different spaces. By enabling transformations that accommodate specific data geometries and constraints, our result increases flexibility in model design and training. We demonstrate this through a reverse-time Itô lemma for transformed diffusion models and a generalized sliced score matching framework. Together, these contributions lay a mathematical foundation for transforming score-based algorithms, opening up new avenues for research and applications. 

\subsection{Related Work}

Our work on score function transformations intersects with several important areas of research in machine learning and statistics. In the realm of distribution transformations, our approach complements existing research such as normalizing flows \cite{rezende2016variational, dinh2017density}, which primarily focus on mapping densities rather than their scores. A  review by \cite{köthe2023reviewchangevariableformulas} highlights the importance and breadth of change of variable formulas in generative modeling. While these works deal mainly with density transformations, our focus on score function transformations offers a  perspective particularly useful in generative diffusion scenarios.

Recent advancements in score-based diffusion models have challenged the conventional use of Gaussian forward and reverse processes in generative modeling. Several studies have explored the application of diffusion models to constrained domains where standard Gaussian diffusion is inadequate. For instance, \cite{de2022riemannian} adapted diffusion processes to manifold structures. Other researchers have investigated the use of reflected Brownian motion for sampling in constrained domains \cite{lou2023reflected, fishman2024diffusion}, further expanding the applicability of diffusion models to complex geometries.

Our chess example, which explores diffusion on the probability simplex by transforming a VP SDE, aligns closely with the work of \cite{floto2023diffusion}. However, while they derive the score and reverse-time process entirely within the more complex simplex space, our approach offers a simpler model that is easier to train and sample from. In a similar vein to our paper, \cite{avdeyev2023dirichlet} employs a change of variables to train a score model with a Jacobi forward diffusion process for biological sequence generation. Our reverse-time Itô lemma generalizes and expands upon these efforts, providing a more comprehensive framework for score transformations.

The field has also witnessed the development of more exotic diffusion processes aimed at improving sample quality. Notable examples include the generative fractional diffusion models introduced by \cite{nobis2023generative}, which extend the concept of diffusion to fractional dynamics, and the critically-damped Langevin diffusion proposed by \cite{dockhorn2021score} for score-based generative modeling, offering enhanced sampling efficiency. The score change of variables framework can potentially be applied to sampling on these diverse domains and processes, offering significant benefits in cases like the simplex diffusion of \cite{floto2023diffusion}. While for some scenarios, no straightforward transformation may exist, we believe this avenue of research—finding transformations to simplify training and sampling in these complex domains—remains largely unexplored and holds considerable promise.

Score matching techniques have been instrumental in estimating and learning score functions, with various approaches developed over the years. Our work on generalized sliced score matching (GSSM) builds upon the original sliced score matching \cite{song2019slicedscorematchingscalable}, extending the concept to allow for more flexible, non-linear projections in the score estimation process. This approach relates to other established methods in the field, such as the foundational score matching \cite{hyvarinen2005estimation}, denoising score matching \cite{vincent2011connection}, and the Stein spectral gradient estimator \cite{shi2018spectral}. Each of these techniques offers unique advantages in different contexts, and the score change of variables framework presented in this paper may open avenues for developing new score matching methods. For instance, it might enable the adaptation of existing techniques to transformed spaces, potentially simplifying score estimation in complex domains or for intricate probability distributions.

 By examining score function transformations and their applications in diffusion models and score matching, this work aims to contribute to the ongoing dialogue in these fields. The framework presented here offers potential enhancements to a range of machine learning tasks involving score functions and generative modeling, bridging theoretical concepts with practical applications. As the field continues to evolve, further investigation into the interplay between these areas may yield additional insights and methodologies for addressing complex problems in score-based machine learning.

\subsection{Limitations}
While our framework provides a general approach for transforming score functions, several limitations warrant consideration:

\begin{itemize}
    \item \textbf{Computational and Numerical Challenges:} The computational cost of Jacobian and Hessian calculations, particularly in the reverse-time Itô lemma, can be significant for complex transformations where analytical solutions are not tractable. While we prove the existence of transformed score functions, our framework does not guarantee numerical stability or efficient convergence for sampling in transformed spaces.
    
    \item \textbf{Restrictive Transformation Requirements:} Our analysis is currently limited to smooth, invertible mappings, excluding a wide class of potentially useful non-differentiable or non-invertible transformations. Exploring the reverse-time Itô lemma with non-invertible transformations, akin to the original Itô lemma, represents an interesting direction for future work.
    
    \item \textbf{Limited Variance Reduction Techniques:} For Generalized Sliced Score Matching (GSSM), current variance reduction techniques are specific to certain families of transformations and often require individual derivation. This limits their general applicability and increases the computational burden for novel transformations.
    
    \item \textbf{Practical Utility:} While we demonstrate potential applications, such as controlling sampling dynamics in the chess example, the broader practical benefits of sampling directly in transformed spaces need further exploration. In many cases, it may be simpler to sample in the original space and transform the results.
\end{itemize}

\subsection{Future Directions}
Our work opens up several promising avenues for future research:

\begin{itemize}
    \item \textbf{Expressive Transformations:} While we focused on analytically tractable mappings, exploring more expressive transformations, such as invertible neural networks, could offer greater flexibility and adaptability to complex data geometries. This data-driven approach could lead to significant performance improvements, albeit at increased computational cost.

    \item \textbf{Theoretical Insights into Diffusion Processes:} Deeper theoretical study is needed to understand the interplay between transformations and diffusion processes. Such insights could guide the selection of optimal transformations, challenging the prevailing assumption that the same space should be used for both training and sampling.

    \item \textbf{Inverse Problem Solving:} The reverse-time Itô lemma presents a promising direction for solving inverse problems. By leveraging the ability to sample in transformed spaces, we can potentially develop more effective methods for inverse problem solving in various domains, building upon recent work in diffusion-based inverse problems \cite{jacobsen2023cocogen, chung2022improving}.

    \item \textbf{Extension to Other Score-Based Methods:} Extending the score change of variables framework to other score-based methods, including variational inference \cite{ranganath2014black} and energy-based models \cite{song2021trainenergybasedmodels}, presents a compelling direction for future research. This could unlock new capabilities within these frameworks and broaden their applicability.

    \item \textbf{Adaptive Transformation Selection:} Developing methods to adaptively select or learn optimal transformations during the training or sampling process could significantly enhance the flexibility and performance of score-based models across diverse applications.
\end{itemize}

In summary, the theoretical foundations established here open the door to more versatile and effective  methodologies in both training of score models and sampling from generative diffusion models. We believe the change of variables perspective will advance the performance, stability, and interpretability of score-based methods across a wide range of applications.

\newpage
\appendix
\onecolumn
\appendix
\section{Proofs}
\label{appendix:proofs}

\subsection[Proof of Theorem: Score Change of Variables in Rn]{Proof of Theorem \ref{thm:multi_dimensional}: Score Change of Variables in \(\mathbb{R}^n\)}
\begin{proof}
\label{appendix:Proof_of_Theorem_1_Multi}
The probability density function \( q(\mathbf{y}) \) of \( \mathbf{Y} \) is related to \( p(\mathbf{x}) \) by the change of variables formula:
\[
q(\mathbf{y}) = p(\phi^{-1}(\mathbf{y})) \left| \det\left( \mathbf{J}_{\phi^{-1}}(\mathbf{y}) \right) \right|.
\]

Taking the gradient of the logarithm of \( q(\mathbf{y}) \), we obtain:
\[
\nabla_{\mathbf{y}} \log q(\mathbf{y}) = \nabla_{\mathbf{y}} \log p(\phi^{-1}(\mathbf{y})) + \nabla_{\mathbf{y}} \log \left| \det\left( \mathbf{J}_{\phi^{-1}}(\mathbf{y}) \right) \right|.
\]

Applying the chain rule to the first term yields:
\[
\nabla_{\mathbf{y}} \log p(\phi^{-1}(\mathbf{y})) = \mathbf{J}_{\phi^{-1}}(\mathbf{y})^\top \cdot \nabla_{\mathbf{x}} \log p(\mathbf{x}) \bigg|_{\mathbf{x}=\phi^{-1}(\mathbf{y})}.
\]

For the second term, we have:
\begin{align}
\nabla_{\mathbf{y}} \log \left| \det\left( \mathbf{J}_{\phi^{-1}}(\mathbf{y}) \right) \right| &= \frac{\nabla_{\mathbf{y}} \det\left( \mathbf{J}_{\phi^{-1}}(\mathbf{y}) \right) }{ \det\left( \mathbf{J}_{\phi^{-1}}(\mathbf{y}) \right) } \\
&= \text{Tr}\left( \mathbf{J}_{\phi^{-1}}(\mathbf{y})^{-1} \cdot \nabla_{\mathbf{y}} \mathbf{J}_{\phi^{-1}}(\mathbf{y}) \right),
\end{align}
where in the second line we use Jacobi's formula for an invertible matrix.

By matrix calculus, we find:
\begin{align}
\text{Tr}\left( \mathbf{J}_{\phi^{-1}}(\mathbf{y})^{-1} \cdot \nabla_{\mathbf{y}} \mathbf{J}_{\phi^{-1}}(\mathbf{y}) \right) &= \sum_{i=1}^n \mathbf{H}_{\phi_i^{-1}} \cdot \frac{\partial \phi}{\partial x_i} \\
&= \nabla\cdot\ \mathbf{J}_{\phi^{-1}}(\phi(\mathbf{x}))^\top,
\end{align}
where \( \mathbf{H}_{\phi_i^{-1}} \) is the Hessian matrix of \( \phi_i^{-1} \).

Combining both terms, the score function is given by:
\[
\nabla_{\mathbf{y}} \log q(\mathbf{y}) = \mathbf{J}_{\phi^{-1}}(\mathbf{y})^\top \cdot \nabla_{\mathbf{x}} \log p(\mathbf{x}) + \nabla\cdot\ \left(\mathbf{J}_{\phi^{-1}}(\phi(\mathbf{x}))^\top\right)\bigg|_{\mathbf{x}=\phi^{-1}(\mathbf{y})} .
\]
\end{proof}

\begin{remark}
In this proof, we utilize the following result for an invertible function \( f: \mathbb{R}^n \to \mathbb{R}^n \) where \( f(\mathbf{x}) = \mathbf{y} \):
\begin{equation}
\nabla_{\mathbf{x}} \log \left| \det\left( \mathbf{J}_{f}(\mathbf{x}) \right) \right| = \nabla\cdot\ \mathbf{J}^\top_{f}(f^{-1}(\mathbf{y})).
\end{equation}

We can write:
\begin{equation}
\mathbf{J}_{f}(f^{-1}(\mathbf{y}))^\top = \left( \left(\mathbf{J}_{f^{-1}}(\mathbf{y}) \right)^{-1}\right)^\top= \frac{1}{\det\left( \mathbf{J}_{f^{-1}}(\mathbf{y}) \right)} \text{Cof}(\mathbf{J}_{f^{-1}}),
\end{equation}
where \(\text{Cof}(\mathbf{J}_{f^{-1}})\) is the cofactor matrix of \(\mathbf{J}_{f^{-1}}\).

For all differentiable \(f\), we have the fact \cite{evans2022partial}:
\begin{equation}
\nabla_{\mathbf{x}}\cdot \left(\text{Cof}(\mathbf{J}_{f}(\mathbf{x})) \right) = 0.
\end{equation}

Using this, we calculate:
\begin{align*}
\nabla_{\mathbf{x}} \log \left| \det\left( \mathbf{J}_{f}(\mathbf{x}) \right) \right| &= \nabla\cdot\ \mathbf{J}_{f}(f^{-1}(\mathbf{y}))^\top \\
&= \nabla\cdot\ \frac{1}{\det\left( \mathbf{J}_{f^{-1}}(\mathbf{y}) \right)} \text{Cof}(\mathbf{J}_{f^{-1}}(\mathbf{y})) \\
&= \frac{1}{\det\left( \mathbf{J}_{f^{-1}}(\mathbf{y}) \right)} \nabla\cdot\ \text{Cof}(\mathbf{J}_{f^{-1}}(\mathbf{y})) + \text{Cof}(\mathbf{J}_{f^{-1}}(\mathbf{y})) \cdot \nabla_{\mathbf{y}} \frac{1}{\det\left( \mathbf{J}_{f^{-1}}(\mathbf{y}) \right)} \\
&= 0 - \frac{1}{\det\left( \mathbf{J}_{f^{-1}}(\mathbf{y}) \right)^2} \text{Cof}(\mathbf{J}_{f^{-1}}(\mathbf{y})) \cdot \nabla_{\mathbf{y}} \det\left( \mathbf{J}_{f^{-1}}(\mathbf{y}) \right) \\
&= -\mathbf{J}^\top_{f} \cdot \nabla_{\mathbf{y}} \log \left| \det\left( \mathbf{J}_{f^{-1}}(\mathbf{y}) \right) \right|.
\end{align*}

From our above analysis, we have for an invertible and twice-differentiable function \( f \):
\begin{equation}
\label{grad_log_det_identity}
\nabla_{\mathbf{x}} \log \left| \det\left( \mathbf{J}_{f}(\mathbf{x}) \right) \right| = -\mathbf{J}_{f}^\top \cdot \nabla_{\mathbf{y}} \log \left| \det\left( \mathbf{J}_{f^{-1}}(\mathbf{y}) \right) \right|.
\end{equation}
\end{remark}

\subsection[Proof of Corollary: Dimension Reduction]{Proof of Corollary \ref{cor:score_change_variables_different_dimensions}: Dimension Reduction \(\mathbb{R}^n \to \mathbb{R}\)}
We assume \( v(\mathbf{x}): U \to U_1 \) where \(U \subset \mathbb{R}^n\) and \(U_1 \subset \mathbb{R}\) is a twice continuously differentiable transformation and that \( \|\nabla v(\mathbf{x})\|_2 \) is non-zero for all \( \mathbf{x} \) in \( U \). Let $y=v(\mathbf{x})$.
\begin{proof}
Consider the transformation
\begin{align*}  
\phi(\mathbf{x})= (x_1, \cdots, x_{i-1}, v(\mathbf{x}), x_{i+1}, \cdots, x_n) = \mathbf{z}.
\end{align*}
Suppose \( \frac{\partial v(\mathbf{x})}{\partial x_i} \) is non-zero.
By the change of variables formula, we have
\begin{align*}
\nabla_{\mathbf{x}} \log p(\mathbf{x}) &= \mathbf{J}^\top_{\phi(\mathbf{x})} \nabla_{\mathbf{z}} \log q(\mathbf{z}) + \nabla_{\mathbf{x}} \log |\det(\mathbf{J}_{\phi}(\mathbf{x}))| \bigg|_{\mathbf{z}=\phi(\mathbf{x})}\\
&= \mathbf{J}^\top_{\phi(\mathbf{x})} \nabla_{\mathbf{z}} \log q(\mathbf{z}) + \nabla_{\mathbf{x}} \log \left| \frac{\partial v(\mathbf{x})}{\partial x_i} \right| \bigg|_{\mathbf{z}=\phi(\mathbf{x})}.
\end{align*}
Then we have 
\begin{align*}
    \frac{\partial}{\partial x_i} \log p(\mathbf{x}) &= \frac{\partial v(\mathbf{x})}{\partial x_i} \frac{\partial}{\partial z_i} \log q(\mathbf{z}) + \frac{\partial}{\partial x_i} \log \left| \frac{\partial v(\mathbf{x})}{\partial x_i} \right|\\
    &= \frac{\partial v(\mathbf{x})}{\partial x_i} \left( \nabla_y \log q(y) + \frac{\partial}{\partial y} \log p((x_j)_{j \neq i} \mid y) \right) + \frac{\partial}{\partial x_i} \log \left| \frac{\partial v(\mathbf{x})}{\partial x_i} \right|.
\end{align*}
Rearranging, we see
\begin{align}
\label{intermediate_eqn}
   \left(\frac{\partial v(\mathbf{x})}{\partial x_i}\right)^2 \nabla_y \log q(y) &=  \frac{\partial v(\mathbf{x})}{\partial x_i} \frac{\partial}{\partial x_i} \log p(\mathbf{x}) \nonumber\\
   \quad &- \left(\frac{\partial v(\mathbf{x})}{\partial x_i}\right)^2 \frac{\partial}{\partial y} \log p((x_j)_{j \neq i} \mid y) - \frac{\partial^2 v(\mathbf{x})}{\partial x_i^2}.
\end{align}
Since \ref{intermediate_eqn} is true for all $i$ where \( \frac{\partial v(\mathbf{x})}{\partial x_i} \) is non-zero and trivially true where \( \frac{\partial v(\mathbf{x})}{\partial x_i}=0 \), we have
\begin{align*}
\|\nabla v(\mathbf{x})\|_2^2 \nabla_y \log q(y) &= \sum_{i=1}^n \frac{\partial v(\mathbf{x})}{\partial x_i} \frac{\partial}{\partial x_i} \log p(\mathbf{x}) - \sum_{i=1}^n \left(\frac{\partial v(\mathbf{x})}{\partial x_i}\right)^2 \frac{\partial}{\partial y} \log p((x_j)_{j \neq i} \mid y) - \sum_{i=1}^n \frac{\partial^2 v(\mathbf{x})}{\partial x_i^2}\\
&= \nabla v(\mathbf{x})^\top \nabla_{\mathbf{x}} \log p(\mathbf{x}) - \Delta v(\mathbf{x}) - \sum_{i=1}^n \left(\frac{\partial v(\mathbf{x})}{\partial x_i}\right)^2 \frac{\partial}{\partial y} \log p((x_j)_{j \neq i} \mid y).
\end{align*}
Finally, assuming \( \|\nabla v(\mathbf{x})\|_2 \) is non-zero we find 
\begin{align*}
\nabla_y \log q(y) &= \frac{1}{\|\nabla v(\mathbf{x})\|_2^2} \left( \nabla v(\mathbf{x})^\top \nabla_{\mathbf{x}} \log p(\mathbf{x}) - \Delta v(\mathbf{x}) - \sum_{i=1}^n \left(\frac{\partial v(\mathbf{x})}{\partial x_i}\right)^2 \frac{\partial}{\partial y} \log p((x_j)_{j \neq i} \mid y) \right).
\end{align*}
\end{proof}

\subsection[Additional Result: Dimension Expansion]{Additional result: Dimension Expansion \(\mathbb{R} \to \mathbb{R}^n\)}

While not used in the main text, we present an additional change of variables formula for completeness.
\begin{theorem}[Change of Variables Formula \(\mathbb{R} \to \mathbb{R}^n\)]
\label{appendix:R_to_Rn}
Let \( v: \mathbb{R} \to \mathbb{R}^n \) be a differentiable transformation. Let \( p(x) \) be the probability density function of a random variable \( X \in \mathbb{R} \), and let \( q(\mathbf{y}) \) be the probability density function of the transformed random variable \( \mathbf{Y} = v(X) \in \mathbb{R}^n \). Assume that \( \frac{dv_i(x)}{dx} \) is non-zero for all \( x \in \mathbb{R} \) and for all \( i \). Then, the score function \( \nabla_{\mathbf{y}} \log q(\mathbf{y}) \) of \( \mathbf{Y} \) can be expressed as follows:
\begin{equation}
\nabla_{\mathbf{y}} \log q(\mathbf{y}) = \frac{1}{\left(\frac{dv(x)}{dx}\right)^2}\left(\frac{dv(x)}{dx} \nabla_x \log p(x) -\frac{d^2v(x)}{dx^2}\right),
\end{equation}
where \( \frac{dv(x)}{dx} \in \mathbb{R}^n\) is the vector of first derivatives and \( \frac{1}{(\frac{dv(x)}{dx})^2} \) denotes element-wise division by the squared components of \( \frac{dv(x)}{dx} \).
\end{theorem}
\begin{proof}
Consider the transformation \( v: \mathbb{R} \to \mathbb{R}^n \) defined by \( \mathbf{y} = v(x) \) where \( v \) is differentiable, and each component \( v_i(x) \) of \( v \) has a non-zero derivative \( \frac{dv_i(x)}{dx} \). For each \( i \), \( v_i \) acts essentially as a scalar transformation from \( \mathbb{R} \to \mathbb{R} \), making \( v_i^{-1} \) well defined.

Applying corollary \ref{cor:score_change_of_variables_1d} to each component \( v_i \) independently, we find the score function for each transformed component \( y_i = v_i(x) \) of the random variable \( \mathbf{Y} \):
\[
\nabla_{y_i} \log q(y_i) = \frac{1}{\left(\frac{dv_i(x)}{dx}\right)^2} \left( \frac{dv_i(x)}{dx} \nabla_x \log p(x) - \frac{d^2 v_i(x)}{dx^2} \right).
\]

For the multidimensional variable \( \mathbf{Y} = v(X) \), the score function \( \nabla_{\mathbf{y}} \log q(\mathbf{y}) \) can be composed by stacking the gradients of each component \( \log q_i(y_i) \), yielding:
\[
\nabla_{\mathbf{y}} \log q(\mathbf{y}) = \left[ \nabla_{y_1} \log q(y_1), \dots, \nabla_{y_n} \log q(y_n) \right]^\top.
\]

Substituting the expression for each \( \nabla_{y_i} \log q(y_i) \) into this vector gives:
\begin{align*}
\nabla_{\mathbf{y}} \log q(\mathbf{y}) = \biggl[ 
    &\frac{1}{\left(\frac{dv_1(x)}{dx}\right)^2} \left( \frac{dv_1(x)}{dx} \nabla_x \log p(x) - \frac{d^2 v_1(x)}{dx^2} \right), \\
    &\ldots, \\
    &\frac{1}{\left(\frac{dv_n(x)}{dx}\right)^2} \left( \frac{dv_n(x)}{dx} \nabla_x \log p(x) - \frac{d^2 v_n(x)}{dx^2} \right) 
\biggr]^\top.
\end{align*}

Combining these component-wise expressions, we obtain:
\[
\nabla_{\mathbf{y}} \log q(\mathbf{y}) = \frac{1}{\left(\frac{dv(x)}{dx}\right)^2}\left(\frac{dv(x)}{dx} \nabla_x \log p(x) - \frac{d^2 v(x)}{dx^2}\right),
\]
where \( \frac{dv(x)}{dx} \in \mathbb{R}^n\) is the vector of first derivatives \( [\frac{dv_1(x)}{dx}, \ldots, \frac{dv_n(x)}{dx}]^\top \), \( \frac{d^2v(x)}{dx^2} \in \mathbb{R}^n\) is the vector of second derivatives, and all operations (division, multiplication) are performed element-wise.
\end{proof}

\section{Derivation of the Reverse-Time Itô Lemma}
\label{appendix:reverse_time_ito_lemma}

This section provides a  derivation of the reverse-time Itô lemma, which establishes the SDE governing the time-reversed transformed process. We consider two complementary approaches: 
\begin{enumerate}
    \item \textbf{Transform-then-Reverse:} We first transform the forward process using Itô's lemma and then apply the time-reversal formula to the transformed process.
    \item \textbf{Reverse-then-Transform:} We first construct the reverse-time SDE of the original process and then transform it using Itô's lemma.
\end{enumerate}

Both approaches lead to the same result, demonstrating the path-wise equivalence:
\[
\text{Reverse}(\phi(\mathbf{X}_t, t)) \equiv \phi(\text{Reverse}(\mathbf{X}_t), t).
\]
We also show  the reverse-time Wiener process defined in \cite{anderson1982reverse} is invariant under invertible  transformation.

\subsection{Proof of Reverse-Time Itô Lemma via Transform-then-Reverse}
\label{app:reverse_ito_proof_explicit}

Consider the forward process $\mathbf{X}$ governed by:
\begin{equation}
\mathrm{d}\mathbf{X} = \mathbf{f}(\mathbf{X},t)\mathrm{d}t + \mathbf{G}(\mathbf{X},t)\mathrm{d}\mathbf{W}
\label{eq:forward_sde_app_explicit}
\end{equation}
and its transformation $\mathbf{Y} = \phi(\mathbf{X},t)$ where $\phi: \mathbb{R}^n \times [0,T] \to \mathbb{R}^n$ is a bijective $C^{2,1}$ transformation.

By Itô's Lemma, the transformed process $\mathbf{Y}$ satisfies:
\begin{equation}
\mathrm{d}\mathbf{Y} = \tilde{\mathbf{f}}(\mathbf{Y},t)\mathrm{d}t + \tilde{\mathbf{G}}(\mathbf{Y},t)\mathrm{d}\mathbf{W}
\label{eq:transformed_forward_app_explicit}
\end{equation}
where for $\mathbf{x} = \phi^{-1}(\mathbf{y},t)$:
\begin{align}
\tilde{\mathbf{f}}(\mathbf{y},t) &= \frac{\partial \phi(\mathbf{x},t)}{\partial t} + \mathbf{J}_\phi(\mathbf{x},t)\mathbf{f}(\mathbf{x},t)  + \frac{1}{2} \text{Tr}\left[ \mathbf{G}(\mathbf{x}, t)^\top \mathbf{H}_{\phi}(\mathbf{x}, t) \mathbf{G}(\mathbf{x}, t)\right]\label{eq:transformed_drift_app_explicit} \\
\tilde{\mathbf{G}}(\mathbf{y},t) &= \mathbf{J}_\phi(\mathbf{x},t)\mathbf{G}(\mathbf{x},t) \label{eq:transformed_diffusion_app_explicit}
\end{align}

To obtain the reverse-time SDE for $\mathbf{Y}$, we apply Anderson's time-reversal formula (Theorem 4.1 in \cite{anderson1982reverse}) to the transformed forward SDE \eqref{eq:transformed_forward_app_explicit}. This yields:

\begin{equation}
\mathrm{d}\bar{\mathbf{Y}} = \left[\tilde{\mathbf{f}}(\bar{\mathbf{Y}},t) - \tilde{\mathbf{G}}(\bar{\mathbf{Y}},t)\tilde{\mathbf{G}}(\bar{\mathbf{Y}},t)^\top\nabla_{\mathbf{y}}\log q_t(\bar{\mathbf{Y}}) - \nabla_{\mathbf{y}}\cdot(\tilde{\mathbf{G}}(\bar{\mathbf{Y}},t)\tilde{\mathbf{G}}(\bar{\mathbf{Y}},t)^\top)\right]\mathrm{d}t + \tilde{\mathbf{G}}(\bar{\mathbf{Y}},t)\mathrm{d}\bar{\mathbf{W}}
\label{eq:raw_reverse_app_explicit}
\end{equation}
where $\mathrm{d}\bar{\mathbf{W}}$ is the infinitesimal increment of a reverse-time Wiener process, and $q_t(\mathbf{y})$ is the probability density of $\mathbf{Y}$.

Using our score change of variables formula (Theorem \ref{thm:multi_dimensional}), we express the score $\nabla_{\mathbf{y}}\log q_t(\mathbf{y})$ in terms of the score in the original space:
\begin{align}
\nabla_{\mathbf{y}}\log q_t(\mathbf{y}) &= \mathbf{J}_{\phi^{-1}}(\mathbf{y},t)^\top\nabla_{\mathbf{x}}\log p_t(\phi^{-1}(\mathbf{y},t)) + \nabla_{\mathbf{x}}\cdot\left(\mathbf{J}_{\phi^{-1}}(\phi(\mathbf{x},t),t)^\top\right)\Big|_{\mathbf{x}=\phi^{-1}(\mathbf{y},t)}
\label{eq:score_transform_app_explicit}
\end{align}

We now substitute \eqref{eq:score_transform_app_explicit} into \eqref{eq:raw_reverse_app_explicit} with $\mathbf{x} = \phi^{-1}(\mathbf{y},t)$ and analyze each term:

\begin{align*}
\mathrm{d}\bar{\mathbf{Y}} &= \left[\tilde{\mathbf{f}}(\bar{\mathbf{Y}},t) - \underbrace{\tilde{\mathbf{G}}(\bar{\mathbf{Y}},t)\tilde{\mathbf{G}}(\bar{\mathbf{Y}},t)^\top\mathbf{J}_{\phi^{-1}}(\bar{\mathbf{Y}},t)^\top\nabla_{\mathbf{x}}\log p_t(\mathbf{x})}_{\text{Term A}} \right. \\
&\quad \left. - \underbrace{\tilde{\mathbf{G}}(\bar{\mathbf{Y}},t)\tilde{\mathbf{G}}(\bar{\mathbf{Y}},t)^\top\nabla_{\mathbf{x}}\cdot\left(\mathbf{J}_{\phi^{-1}}(\phi(\bar{\mathbf{X}},t),t)^\top\right)}_{\text{Term B}} - \underbrace{\nabla_{\mathbf{y}}\cdot(\tilde{\mathbf{G}}(\bar{\mathbf{Y}},t)\tilde{\mathbf{G}}(\bar{\mathbf{Y}},t)^\top)}_{\text{Term C}}\right]\mathrm{d}t + \tilde{\mathbf{G}}(\bar{\mathbf{Y}},t)\mathrm{d}\bar{\mathbf{W}}
\end{align*}

\textbf{Term A:} Using $\mathbf{J}_{\phi^{-1}}(\mathbf{y},t) = [\mathbf{J}_\phi(\mathbf{x},t)]^{-1}$ and \eqref{eq:transformed_diffusion_app_explicit}, we simplify Term A:
\begin{align*}
\text{Term A} &= \mathbf{J}_\phi(\mathbf{x},t)\mathbf{G}(\mathbf{x},t)\mathbf{G}(\mathbf{x},t)^\top\mathbf{J}_\phi(\mathbf{x},t)^\top[\mathbf{J}_\phi(\mathbf{x},t)^{-\top}]\nabla_{\mathbf{x}}\log p_t(\mathbf{x}) \\
&= \mathbf{J}_\phi(\mathbf{x},t)\mathbf{G}(\mathbf{x},t)\mathbf{G}(\mathbf{x},t)^\top\nabla_{\mathbf{x}}\log p_t(\mathbf{x})
\end{align*}

\textbf{Term B:} Using identity \eqref{grad_log_det_identity}, we rewrite Term B:
\begin{align*}
\text{Term B} &= \mathbf{J}_\phi(\mathbf{x},t)\mathbf{G}(\mathbf{x},t)\mathbf{G}(\mathbf{x},t)^\top\mathbf{J}_\phi(\mathbf{x},t)^\top\nabla_{\mathbf{x}}\cdot\left(\mathbf{J}_{\phi^{-1}}(\phi(\mathbf{x},t),t)^\top\right) \\
&= -\mathbf{J}_\phi(\mathbf{x},t)\mathbf{G}(\mathbf{x},t)\mathbf{G}(\mathbf{x},t)^\top\nabla_{\mathbf{y}}\cdot\left(\mathbf{J}_{\phi}(\phi^{-1}(\mathbf{y},t),t)^\top\right)
\end{align*}

\textbf{Term C:} Applying the divergence operator to the product $\tilde{\mathbf{G}}(\mathbf{y}, t)\tilde{\mathbf{G}}(\mathbf{y}, t)^\top$ yields:
\begin{align*}
\text{Term C} &= \nabla_{\mathbf{y}} \cdot \bigl(\tilde{\mathbf{G}}(\mathbf{y}, t)\tilde{\mathbf{G}}(\mathbf{y}, t)^\top \bigr)\\
&=\nabla_{\mathbf{y}} [\mathbf{J}_{\phi}(\phi^{-1}(\mathbf{y},t), t) \mathbf{G}(\phi^{-1}(\mathbf{y},t), t)\mathbf{G}(\phi^{-1}(\mathbf{y},t), t)^\top] : [\mathbf{J}_{\phi}(\phi^{-1}(\mathbf{y},t), t)]^T\\
&\quad + \mathbf{J}_{\phi}(\mathbf{x}, t) \mathbf{G}(\mathbf{x}, t)\mathbf{G}(\mathbf{x}, t)^\top \nabla_{\mathbf{y}}\cdot\left(\mathbf{J}_{\phi}(\phi^{-1}(\mathbf{y}, t), t)^\top\right)
\end{align*}
where here and throughout the paper we use the notation $\mathbf{A}:\mathbf{B}$ to denote a double contraction between a third-order tensor $\mathbf{A}$ and a matrix $\mathbf{B}$. Specifically, if $\mathbf{A}$ has elements $A_{ijk}$ and $\mathbf{B}$ has elements $B_{kl}$, then $\mathbf{A}:\mathbf{B}$ results in a vector $\mathbf{C}$ with elements $C_i$ defined as:
\begin{equation}
C_i = \sum_{j,k} A_{ijk} B_{jk}
\end{equation}

Combining Terms B and C, we get:
\begin{align*}
\text{Term B} + \text{Term C} &= \nabla_{\mathbf{y}} [\mathbf{J}_{\phi}(\phi^{-1}(\mathbf{y},t), t) \mathbf{G}(\phi^{-1}(\mathbf{y},t), t)\mathbf{G}(\phi^{-1}(\mathbf{y},t), t)^\top] : [\mathbf{J}_{\phi}(\phi^{-1}(\mathbf{y},t), t)]^T\\
&\quad + \mathbf{J}_{\phi}(\mathbf{x}, t) \mathbf{G}(\mathbf{x}, t)\mathbf{G}(\mathbf{x}, t)^\top \nabla_{\mathbf{y}}\cdot\left(\mathbf{J}_{\phi}(\phi^{-1}(\mathbf{y}, t), t)^\top\right)\\
&\quad -\mathbf{J}_\phi(\mathbf{x},t)\mathbf{G}(\mathbf{x},t)\mathbf{G}(\mathbf{x},t)^\top\nabla_{\mathbf{y}}\cdot\left(\mathbf{J}_{\phi}(\phi^{-1}(\mathbf{y},t),t)^\top\right)\\
&=\nabla_{\mathbf{y}} [\mathbf{J}_{\phi}(\phi^{-1}(\mathbf{y},t), t) \mathbf{G}(\phi^{-1}(\mathbf{y},t), t)\mathbf{G}(\phi^{-1}(\mathbf{y},t), t)^\top] : [\mathbf{J}_{\phi}(\phi^{-1}(\mathbf{y},t), t)]^T
\end{align*}

To further simplify this expression, we use the following lemma:

\textbf{Lemma:}
\begin{align*}
&\quad \nabla_{\mathbf{y}} [\mathbf{J}_{\phi}(\phi^{-1}(\mathbf{y},t), t) \mathbf{G}(\phi^{-1}(\mathbf{y},t), t)\mathbf{G}(\phi^{-1}(\mathbf{y},t), t)^\top] : [\mathbf{J}_\phi(\phi^{-1}(\mathbf{y},t), t)]^T\\
&=\text{Tr}\left[ \mathbf{G}(\mathbf{x}, t)^\top \mathbf{H}_{\phi}(\mathbf{x}, t) \mathbf{G}(\mathbf{x}, t) \right]+ \mathbf{J}_{\phi}(\mathbf{x}, t)\nabla_{\mathbf{x}} \cdot \left(\mathbf{G}(\mathbf{x}, t)\mathbf{G}(\mathbf{x}, t)^\top\right)
\end{align*}
\textit{Proof:} See Appendix~\ref{app:lemma_proof}.

Applying the lemma, we obtain:
\begin{align*}
\text{Term B} + \text{Term C} &= \text{Tr}\left[ \mathbf{G}(\mathbf{x}, t)^\top \mathbf{H}_{\phi}(\mathbf{x}, t) \mathbf{G}(\mathbf{x}, t) \right]+ \mathbf{J}_{\phi}(\mathbf{x}, t)\nabla_{\mathbf{x}} \cdot \left(\mathbf{G}(\mathbf{x}, t)\mathbf{G}(\mathbf{x}, t)^\top\right)
\end{align*}

Substituting Terms A, B, and C back into the reverse SDE and using the definition of $\tilde{\mathbf{f}}(\mathbf{y}, t)$ from \eqref{eq:transformed_drift_app_explicit}, we obtain the reverse drift for $\bar{\mathbf{Y}}_t$:
\begin{align*}
\hat{\mathbf{f}}(\mathbf{y},t) &= \tilde{\mathbf{f}}(\mathbf{y},t) - \left(\text{Term A + Term B + Term C}\right)\\
&=\frac{\partial \phi(\mathbf{x}, t)}{\partial t} + \mathbf{J}_{\phi}(\mathbf{x}, t) \mathbf{f}(\mathbf{x}, t)+ \frac{1}{2} \text{Tr}\left[ \mathbf{G}(\mathbf{x}, t)^\top \mathbf{H}_{\phi}(\mathbf{x}, t) \mathbf{G}(\mathbf{x}, t) \right]\\
&\quad-\big(\mathbf{J}_\phi(\mathbf{x},t)\mathbf{G}(\mathbf{x},t)\mathbf{G}(\mathbf{x},t)^\top\nabla_{\mathbf{x}}\log p_t(\mathbf{x})+\text{Tr}\left[ \mathbf{G}(\mathbf{x}, t)^\top \mathbf{H}_{\phi}(\mathbf{x}, t) \mathbf{G}(\mathbf{x}, t) \right]\\
&\qquad + \mathbf{J}_{\phi}(\mathbf{x}, t)\nabla_{\mathbf{x}} \cdot \left(\mathbf{G}(\mathbf{x}, t)\mathbf{G}(\mathbf{x}, t)^\top\right)\big)\\
&=\frac{\partial \phi(\mathbf{x}, t)}{\partial t}+\mathbf{J}_\phi(\mathbf{x},t)\left(\mathbf{f}(\mathbf{x}, t)-\mathbf{G}(\mathbf{x},t)\mathbf{G}(\mathbf{x},t)^\top\nabla_{\mathbf{x}}\log p_t(\mathbf{x})-\nabla_{\mathbf{x}} \cdot \left(\mathbf{G}(\mathbf{x}, t)\mathbf{G}(\mathbf{x}, t)^\top\right)\right)\\
&\quad -\frac{1}{2} \text{Tr}\left[ \mathbf{G}(\mathbf{x}, t)^\top \mathbf{H}_{\phi}(\mathbf{x}, t) \mathbf{G}(\mathbf{x}, t)\right]\\
&=\frac{\partial \phi(\mathbf{x}, t)}{\partial t} + \mathbf{J}_{\phi}(\mathbf{x}, t) \bar{\mathbf{f}}(\mathbf{x}, t) - \frac{1}{2} \text{Tr}\left[ \mathbf{G}(\mathbf{x}, t)^\top \mathbf{H}_{\phi}(\mathbf{x}, t) \mathbf{G}(\mathbf{x}, t) \right]
\end{align*}
where $\bar{\mathbf{f}}(\mathbf{x}, t) = \mathbf{f}(\mathbf{x}, t)-\mathbf{G}(\mathbf{x},t)\mathbf{G}(\mathbf{x},t)^\top\nabla_{\mathbf{x}}\log p_t(\mathbf{x})-\nabla_{\mathbf{x}} \cdot \left(\mathbf{G}(\mathbf{x}, t)\mathbf{G}(\mathbf{x}, t)^\top\right)$ is the reverse drift of the original process $\mathbf{X}_t$.
We conclude
\begin{align}
\mathrm{d}\bar{\mathbf{Y}} =  \hat{\mathbf{f}}(\bar{\mathbf{Y}}, t)dt  + \tilde{\mathbf{G}}(\mathbf{\bar{Y}}, t)\mathrm{d}\bar{\mathbf{W}}(t).
\end{align}
which matches \eqref{eq:reverse_sde_Y_final}.

\subsection{Proof of Reverse-Time Itô Lemma via Reverse-then-Transform}
\label{app:reverse_then_transform}

Now, we derive the reverse-time SDE by first reversing the original process $\mathbf{X}$ and then applying the transformation $\phi$. Let $\bar{\mathbf{X}}$ denote the reverse process of $\mathbf{X}$, which satisfies:
\begin{align}
\mathrm{d}\bar{\mathbf{X}} &=\bar{\mathbf{f}}(\bar{\mathbf{X}},t) \mathrm{d}t+ \mathbf{G}(\bar{\mathbf{X}},t)\mathrm{d}\bar{\mathbf{W}}
\end{align}
where 
\begin{align}
\bar{\mathbf{f}}(\mathbf{x},t)&=\mathbf{f}(\mathbf{x},t) - \mathbf{G}(\mathbf{x},t)\mathbf{G}(\mathbf{x},t)^\top\nabla_{\mathbf{x}}\log p_t(\mathbf{x})  - \nabla_{\mathbf{x}}\cdot(\mathbf{G}(\mathbf{x},t)\mathbf{G}(\mathbf{x},t)^\top) \label{eq:reverse_sde_X_app} 
\end{align}
and $\mathrm{d}\bar{\mathbf{W}}$ is the infinitesimal increment of a reverse-time Wiener process. 

To apply Itô's lemma, we re-parameterize the reverse-time SDE using forward time $\tau = T-t$, which gives $\mathrm{d}\tau = -\mathrm{d}t$. The SDE in forward time becomes:
\begin{align}
\mathrm{d}\bar{\mathbf{X}} =-\bar{\mathbf{f}}(\bar{\mathbf{X}},\tau) \mathrm{d}\tau+ \mathbf{G}(\bar{\mathbf{X}},\tau)\mathrm{d}\bar{\mathbf{W}}
\end{align}

Now, let $\bar{\mathbf{Y}} = \phi(\bar{\mathbf{X}},\tau)$. Applying Itô's lemma to this transformation, we obtain:
\begin{align}
\mathrm{d}\bar{\mathbf{Y}} &= \left[\frac{\partial\phi(\bar{\mathbf{X}},\tau)}{\partial\tau} - \mathbf{J}_\phi(\bar{\mathbf{X}},\tau)\bar{\mathbf{f}}(\bar{\mathbf{X}},\tau) + \frac{1}{2}\mathrm{Tr}\left(\mathbf{G}(\bar{\mathbf{X}},\tau)^\top\mathbf{H}_\phi(\bar{\mathbf{X}},\tau)\mathbf{G}(\bar{\mathbf{X}},\tau)\right)\right]\mathrm{d}\tau \nonumber \\
&\quad + \mathbf{J}_\phi(\bar{\mathbf{X}},\tau)\mathbf{G}(\bar{\mathbf{X}},\tau)\mathrm{d}\bar{\mathbf{W}} \label{eq:transformed_reverse_raw_app}
\end{align}

Substituting back into the reverse time $t=T-\tau$, we get:
\begin{align}
\mathrm{d}\bar{\mathbf{Y}}&= \left[\frac{\partial\phi(\bar{\mathbf{X}},t)}{\partial t} + \mathbf{J}_\phi(\bar{\mathbf{X}},t)\bar{\mathbf{f}}(\bar{\mathbf{X}},t) - \frac{1}{2}\mathrm{Tr}\left(\mathbf{G}(\bar{\mathbf{X}},t)^\top\mathbf{H}_\phi(\bar{\mathbf{X}},t)\mathbf{G}(\bar{\mathbf{X}},t)\right)\right]\mathrm{d}t \nonumber \\
&\quad + \mathbf{J}_\phi(\bar{\mathbf{X}},t)\mathbf{G}(\bar{\mathbf{X}},t)\mathrm{d}\bar{\mathbf{W}} \label{eq:transformed_reverse_time_app}
\end{align}

This is the same result we obtained in Appendix~\ref{app:reverse_ito_proof_explicit} using the transform-then-reverse approach.

\subsection{Invariance of the Reverse-Time Wiener Process}
\label{app:wiener_invariance}

We show that the reverse-time Wiener process is invariant under a smooth transformation $\phi$, meaning that the same $\mathrm{d}\bar{\mathbf{W}}$ appears in both the reverse SDE for $\mathbf{X}$ and the reverse SDE for $\mathbf{Y} = \phi(\mathbf{X}, t)$.

Recall from Anderson \cite{anderson1982reverse} that the reverse-time Wiener process increment $\mathrm{d}\bar{\mathbf{W}}$ for the original process $\mathbf{X}$ is given by:
\begin{equation}
    \mathrm{d}\bar{\mathbf{W}} = \mathrm{d}\mathbf{W} + \frac{1}{p_t(\mathbf{x})} \nabla_{\mathbf{x}} \cdot \left[ p_t(\mathbf{x})  \mathbf{G}(\mathbf{x}, t)^\top \right] \mathrm{d}t.
    \label{eq:dw_tilde_x}
\end{equation}
Similarly, for the transformed process $\mathbf{Y}$, the reverse-time Wiener process increment is given by:
\begin{equation}
    \mathrm{d}\bar{\mathbf{W}} = \mathrm{d}\mathbf{W} + \frac{1}{q_t(\mathbf{y})} \nabla_{\mathbf{y}} \cdot \left[ q_t(\mathbf{y})  \tilde{\mathbf{G}}(\mathbf{y}, t)^\top \right] \mathrm{d}t.
    \label{eq:dw_tilde_y}
\end{equation}

We observe:

\begin{align*}
    \frac{1}{q_t(\mathbf{y})} \nabla_{\mathbf{y}} \cdot \left[  \tilde{\mathbf{G}}(\mathbf{y}, t)^\top q_t(\mathbf{y})\right]&=\tilde{\mathbf{G}}(\mathbf{y}, t)^\top \nabla_{\mathbf{y}}\log q_t(\mathbf{y})+\nabla_{\mathbf{y}} \cdot  \tilde{\mathbf{G}}(\mathbf{y}, t)^\top\\
    &=\tilde{\mathbf{G}}(\mathbf{y}, t)^\top \left(\mathbf{J}_{\phi^{-1}}(\mathbf{y}, t)^\top \nabla_{\mathbf{x}} \log p(\mathbf{x})+ \nabla_{\mathbf{x}}\cdot \left(\mathbf{J}_{\phi^{-1}}(\phi(\mathbf{x},t),t)^\top\right)\right)\\
    &\quad +\nabla_{\mathbf{y}} \cdot  \tilde{\mathbf{G}}(\mathbf{y}, t)^\top\bigg|_{\mathbf{x}=\phi^{-1}(\mathbf{y})}\\
    &=\mathbf{G}(\mathbf{x}, t)^\top\left(\nabla_{\mathbf{x}} \log p(\mathbf{x})- \nabla_{\mathbf{y}}\cdot \left(\mathbf{J}_{\phi}(\phi^{-1}(\mathbf{y}, t), t)^\top\right)\right)\\
    &\quad +\nabla_{\mathbf{y}} \cdot  \tilde{\mathbf{G}}(\mathbf{y}, t)^\top\bigg|_{\mathbf{x}=\phi^{-1}(\mathbf{y}, t)}
\end{align*}
Where we used the score change of variables \ref{thm:multi_dimensional} and in the last line we use the identity \eqref{grad_log_det_identity}: \begin{equation}\mathbf{J}_{\phi}(\mathbf{x}, t)^\top\nabla_{\mathbf{x}}\cdot \left(\mathbf{J}_{\phi^{-1}}(\phi(\mathbf{x},t),t)^\top\right)\bigg|_{\mathbf{x}=\phi^{-1}(\mathbf{y}, t)}=- \nabla_{\mathbf{y}}\cdot \left(\mathbf{J}_{\phi}(\phi^{-1}(\mathbf{y}, t), t)^\top\right)
\end{equation}

Now we have through matrix calculus identities:

\begin{align*}
    \nabla_{\mathbf{y}} \cdot  \tilde{\mathbf{G}}(\mathbf{y}, t)^\top&=\nabla_{\mathbf{y}}\mathbf{G}(\phi^{-1}(\mathbf{y}, t), t)^\top:\mathbf{J}_{\phi}(\mathbf{x}, t)^\top +  \mathbf{G}(\mathbf{x}, t)^\top\nabla_{\mathbf{y}}\cdot \mathbf{J}_{\phi}(\phi^{-1}(\mathbf{y}, t), t)^\top\bigg|_{\mathbf{x}=\phi^{-1}(\mathbf{y}, t)}\\
    &=\nabla_{\mathbf{x}}\mathbf{G}(\mathbf{x} , t)^\top\mathbf{J}_{\phi^{-1}}(\mathbf{y}, t):\mathbf{J}_{\phi}(\mathbf{x}, t)^\top +  \mathbf{G}(\mathbf{x}, t)^\top\nabla_{\mathbf{y}}\cdot \mathbf{J}_{\phi}(\phi^{-1}(\mathbf{y}, t), t)^\top\bigg|_{\mathbf{x}=\phi^{-1}(\mathbf{y}, t)}\\
    &=\nabla_{\mathbf{x}}\mathbf{G}(\mathbf{x} , t)^\top:\mathbf{I} +  \mathbf{G}(\mathbf{x}, t)^\top\nabla_{\mathbf{y}}\cdot \mathbf{J}_{\phi}(\phi^{-1}(\mathbf{y}, t), t)^\top\bigg|_{\mathbf{x}=\phi^{-1}(\mathbf{y}, t)}\\
    &=\nabla_{\mathbf{x}}\cdot \mathbf{G}(\mathbf{x} , t)^\top +  \mathbf{G}(\mathbf{x}, t)^\top\nabla_{\mathbf{y}}\cdot \mathbf{J}_{\phi}(\phi^{-1}(\mathbf{y}, t), t)^\top\bigg|_{\mathbf{x}=\phi^{-1}(\mathbf{y}, t)}
\end{align*}

Thus we find:
\begin{align*}
    \frac{1}{q_t(\mathbf{y})} \nabla_{\mathbf{y}} \cdot \left[  \tilde{\mathbf{G}}(\mathbf{y}, t)^\top q_t(\mathbf{y})\right]&=\mathbf{G}(\mathbf{x}, t)^\top\left(\nabla_{\mathbf{x}} \log p(\mathbf{x})- \nabla_{\mathbf{y}}\cdot \left(\mathbf{J}_{\phi}(\phi^{-1}(\mathbf{y}, t), t)^\top\right)\right)\\
    &\quad+\nabla_{\mathbf{y}} \cdot  \tilde{\mathbf{G}}(\mathbf{y}, t)^\top\bigg|_{\mathbf{x}=\phi^{-1}(\mathbf{y}, t)}\\
    &=\mathbf{G}(\mathbf{x}, t)^\top \nabla_{\mathbf{x}} \log p(\mathbf{x})+\nabla_{\mathbf{x}}\cdot \mathbf{G}(\mathbf{x} , t)^\top\bigg|_{\mathbf{x}=\phi^{-1}(\mathbf{y}, t)}\\
    &=\frac{1}{p_t(\mathbf{x})} \nabla_{\mathbf{x}} \cdot \left[ p_t(\mathbf{x})  \mathbf{G}(\mathbf{x}, t)^\top\right]\bigg|_{\mathbf{x}=\phi^{-1}(\mathbf{y}, t)}
\end{align*}

\subsection{Proof of Lemma}
\label{app:lemma_proof}

\textbf{Lemma:}
Let $\phi(\mathbf{x}, t)$ be a twice differentiable and invertible map, with $\mathbf{J}_{\phi}(\mathbf{x}, t)$ denoting its Jacobian and $\mathbf{H}_{\phi}(\mathbf{x}, t)$ its Hessian tensor. Let $\mathbf{G}(\mathbf{x}, t)$ be a matrix-valued function that depends smoothly on $\mathbf{x}$ and $t$. Define $\mathbf{y} = \phi(\mathbf{x}, t)$ and $\mathbf{x} = \phi^{-1}(\mathbf{y}, t)$. Then the following identity holds:
\begin{align*}
\nabla_{\mathbf{y}} \!\Big[\mathbf{J}_{\phi}\bigl(\phi^{-1}(\mathbf{y},t),t\bigr)\,\mathbf{G}\bigl(\phi^{-1}(\mathbf{y},t),t\bigr)\,\mathbf{G}\bigl(\phi^{-1}(\mathbf{y},t),t\bigr)^\top\Big]
\,:\,\bigl[\mathbf{J}_{\phi}\bigl(\phi^{-1}(\mathbf{y},t),t\bigr)\bigr]^\top \nonumber \\
= \mathrm{Tr}\Big[\mathbf{G}(\mathbf{x},t)^\top\,\mathbf{H}_{\phi}(\mathbf{x},t)\,\mathbf{G}(\mathbf{x},t)\Big]
+ \mathbf{J}_{\phi}(\mathbf{x},t)\;\nabla_{\mathbf{x}}\cdot\Big(\mathbf{G}(\mathbf{x},t)\,\mathbf{G}(\mathbf{x},t)^\top\Big)\Big|_{\mathbf{x}=\phi^{-1}(\mathbf{y},t)}.
\end{align*}

\begin{proof}
Define:
\[
\widetilde{\mathbf{A}}(\mathbf{x}, t) = \mathbf{J}_{\phi}(\mathbf{x}, t)\,\mathbf{G}(\mathbf{x}, t)\,\mathbf{G}(\mathbf{x}, t)^\top,
\]
and let:
\[
\mathbf{A}(\mathbf{y}, t) = \widetilde{\mathbf{A}}(\phi^{-1}(\mathbf{y}, t), t).
\]
We want to compute the double contraction $\nabla_{\mathbf{y}} \mathbf{A}(\mathbf{y}, t) : \bigl[\mathbf{J}_{\phi}(\phi^{-1}(\mathbf{y}, t), t)\bigr]^\top$.

First, we compute $\nabla_{\mathbf{y}} \mathbf{A}(\mathbf{y}, t)$ using the chain rule:
\begin{align}
\nabla_{\mathbf{y}} \mathbf{A}(\mathbf{y}, t) &= \nabla_{\mathbf{x}} \widetilde{\mathbf{A}}(\mathbf{x}, t) \cdot \mathbf{J}_{\phi^{-1}}(\mathbf{y}, t) \Big|_{\mathbf{x} = \phi^{-1}(\mathbf{y}, t)}
\end{align}
where $\mathbf{J}_{\phi^{-1}}(\mathbf{y}, t) = [\mathbf{J}_{\phi}(\mathbf{x}, t)]^{-1}$ is the Jacobian of the inverse transformation.

Now, consider the double contraction:
\begin{align}
\nabla_{\mathbf{y}} \mathbf{A}(\mathbf{y}, t) : \bigl[\mathbf{J}_{\phi}(\phi^{-1}(\mathbf{y}, t), t)\bigr]^\top &= \bigl(\nabla_{\mathbf{x}} \widetilde{\mathbf{A}}(\mathbf{x}, t) \cdot \mathbf{J}_{\phi^{-1}}(\mathbf{y}, t)\bigr) : \bigl[\mathbf{J}_{\phi}(\mathbf{x}, t)\bigr]^\top \Big|_{\mathbf{x} = \phi^{-1}(\mathbf{y}, t)} \nonumber \\
&= \nabla_{\mathbf{x}} \widetilde{\mathbf{A}}(\mathbf{x}, t) : \bigl(\mathbf{J}_{\phi^{-1}}(\mathbf{y}, t) \cdot [\mathbf{J}_{\phi}(\mathbf{x}, t)]\bigr) \Big|_{\mathbf{x} = \phi^{-1}(\mathbf{y}, t)} \nonumber \\
&=\nabla_{\mathbf{x}} \widetilde{\mathbf{A}}(\mathbf{x}, t) : \mathbf{I}\Big|_{\mathbf{x} = \phi^{-1}(\mathbf{y}, t)} \nonumber \\
&= \nabla_{\mathbf{x}} \cdot \widetilde{\mathbf{A}}(\mathbf{x}, t) \Big|_{\mathbf{x} = \phi^{-1}(\mathbf{y}, t)}. \label{eq:double_contraction_simplified}
\end{align}

Next, we substitute the definition of $\widetilde{\mathbf{A}}(\mathbf{x}, t)$ and apply the product rule for the divergence:
\begin{align}
\nabla_{\mathbf{x}} \cdot \widetilde{\mathbf{A}}(\mathbf{x}, t) &= \nabla_{\mathbf{x}} \cdot \bigl(\mathbf{J}_{\phi}(\mathbf{x}, t)\,\mathbf{G}(\mathbf{x}, t)\,\mathbf{G}(\mathbf{x}, t)^\top\bigr) \nonumber \\
&= (\nabla_{\mathbf{x}}\mathbf{J}_{\phi}(\mathbf{x}, t)):\bigl(\mathbf{G}(\mathbf{x}, t)\,\mathbf{G}(\mathbf{x}, t)^\top\bigr) + \mathbf{J}_{\phi}(\mathbf{x}, t) \cdot \nabla_{\mathbf{x}} \cdot \bigl(\mathbf{G}(\mathbf{x}, t)\,\mathbf{G}(\mathbf{x}, t)^\top\bigr). \label{eq:divergence_product_rule}
\end{align}

We can rewrite the first term on the right-hand side of \eqref{eq:divergence_product_rule} using the definition of the Hessian tensor $\mathbf{H}_{\phi}$:
\begin{align}
(\nabla_{\mathbf{x}}\mathbf{J}_{\phi}(\mathbf{x}, t)):\bigl(\mathbf{G}(\mathbf{x}, t)\,\mathbf{G}(\mathbf{x}, t)^\top\bigr) = \mathrm{Tr}\bigl[\mathbf{G}(\mathbf{x}, t)^\top \mathbf{H}_{\phi}(\mathbf{x}, t) \mathbf{G}(\mathbf{x}, t)\bigr].
\end{align}

Substituting this back into \eqref{eq:divergence_product_rule}, we get:
\begin{align}
\nabla_{\mathbf{x}} \cdot \widetilde{\mathbf{A}}(\mathbf{x}, t) = \mathrm{Tr}\bigl[\mathbf{G}(\mathbf{x}, t)^\top \mathbf{H}_{\phi}(\mathbf{x}, t) \mathbf{G}(\mathbf{x}, t)\bigr] + \mathbf{J}_{\phi}(\mathbf{x}, t) \cdot \nabla_{\mathbf{x}} \cdot \bigl(\mathbf{G}(\mathbf{x}, t)\,\mathbf{G}(\mathbf{x}, t)^\top\bigr).
\end{align}

Finally, substituting this into \eqref{eq:double_contraction_simplified} and evaluating at $\mathbf{x} = \phi^{-1}(\mathbf{y}, t)$, we arrive at the desired identity:
\begin{align}
\nabla_{\mathbf{y}} \mathbf{A}(\mathbf{y}, t) : \bigl[\mathbf{J}_{\phi}(\phi^{-1}(\mathbf{y}, t), t)\bigr]^\top &= \mathrm{Tr}\bigl[\mathbf{G}(\mathbf{x}, t)^\top \mathbf{H}_{\phi}(\mathbf{x}, t) \mathbf{G}(\mathbf{x}, t)\bigr] \nonumber \\
&\quad + \mathbf{J}_{\phi}(\mathbf{x}, t) \cdot \nabla_{\mathbf{x}} \cdot \bigl(\mathbf{G}(\mathbf{x}, t)\,\mathbf{G}(\mathbf{x}, t)^\top\bigr) \Big|_{\mathbf{x} = \phi^{-1}(\mathbf{y}, t)}.
\end{align}
\end{proof}

\subsection{Derivation of Weighted Loss Function}
\label{appendix:training}

We show that minimizing the weighted loss \eqref{eq:weighted_loss_multidim} yields a score network that correctly approximates the transformed score. Let $q_t(\mathbf{y})$ be the density of the transformed variable $\mathbf{Y} = \phi(\mathbf{X}, t)$.

First, recall from Theorem \ref{thm:multi_dimensional} that:
\begin{equation}
\nabla_{\mathbf{y}} \log q_t(\mathbf{y}) = \mathbf{J}_{\phi^{-1}}(\mathbf{y}, t)^\top \nabla_{\mathbf{x}} \log p_t(\mathbf{x}) + \nabla_{\mathbf{x}}\cdot \left(\mathbf{J}_{\phi^{-1}}(\phi(\mathbf{x}, t),  t)^\top\right)\big|_{\mathbf{x}=\phi^{-1}(\mathbf{y}, t)}.
\end{equation}

Consider the standard denoising score matching objective in the transformed space:
\begin{equation}
\mathbb{E}_{t} \left[ \lambda(t) \, \mathbb{E}_{\mathbf{Y}(0)} \, \mathbb{E}_{\mathbf{Y} | \mathbf{Y}(0)} \left\| \tilde{s}_\theta(\mathbf{Y}, t) - \nabla_{\mathbf{y}} \log q_{0t}(\mathbf{Y} | \mathbf{Y}(0)) \right\|^2 \right],
\end{equation}
where $\tilde{s}_\theta$ is a score network in the transformed space and  $t$ is sampled uniformly from $[0,1]$.

For this to be equivalent to training in the original space, we want:
\begin{equation}
\tilde{s}_\theta(\mathbf{y}, t) = \mathbf{J}_{\phi^{-1}}(\mathbf{y}, t)^\top s_\theta(\phi^{-1}(\mathbf{y}, t), t) + \nabla_{\mathbf{x}}\cdot \left(\mathbf{J}_{\phi^{-1}}(\phi(\mathbf{x}, t), t) ^\top\right),
\end{equation}
where $s_\theta$ is trained in the original space.

Substituting this relationship and changing variables to $\mathbf{X}$:
\begin{align*}
&\mathbb{E}_{t} \Big[ \lambda(t) \, \mathbb{E}_{\mathbf{X}(0)} \, \mathbb{E}_{\mathbf{X} | \mathbf{X}(0)} \Big\| \mathbf{J}_{\phi^{-1}}(\phi(\mathbf{X}, t))^\top s_\theta(\mathbf{X}, t) + \nabla_{\mathbf{x}}\cdot \left(\mathbf{J}_{\phi^{-1}}(\phi(\mathbf{x}, t), t) ^\top\right) \\
&\qquad\qquad - \mathbf{J}_{\phi^{-1}}(\phi(\mathbf{X}, t))^\top \nabla_{\mathbf{x}} \log p_{0t}(\mathbf{X} | \mathbf{X}(0)) - \nabla_{\mathbf{x}}\cdot \left(\mathbf{J}_{\phi^{-1}}(\phi(\mathbf{x}, t), t) ^\top\right) \Big\|^2 \Big] \\
&= \mathbb{E}_{t} \Big[ \lambda(t) \, \mathbb{E}_{\mathbf{X}(0)} \, \mathbb{E}_{\mathbf{X} | \mathbf{X}(0)} \Big\| \mathbf{J}_{\phi^{-1}}(\phi(\mathbf{X}, t))^\top \left( s_\theta(\mathbf{X}, t) - \nabla_{\mathbf{x}} \log p_{0t}(\mathbf{X} | \mathbf{X}(0)) \right) \Big\|^2 \Big],
\end{align*}
which is precisely our weighted objective \eqref{eq:weighted_loss_multidim}.

Therefore, minimizing this weighted loss in the original space is equivalent to training a score network directly in the transformed space. The Jacobian weights ensure that errors in score approximation transform appropriately under $\phi$.

\section{Derivation of the Generalized Sliced Score Matching Objective}
\label{appendix:gssm_derivation}

The derivation of Generalized Sliced Score Matching (GSSM) follows the principles established in the original score matching \citep{hyvarinen2005estimation} and sliced score matching \citep{song2019slicedscorematchingscalable} papers. We extend these ideas to account for general smooth functions.

\begin{assumption}[Integrability Conditions]
\label{assumption:integrability}
For the data distribution $p_d(\mathbf{x})$ and score approximation $s_\theta(\mathbf{x})$:
\begin{enumerate}
    \item $\mathbb{E}_{p_d}[\|\nabla_{\mathbf{x}} \log p_d(\mathbf{X})\|^2] < \infty$
    \item $\mathbb{E}_{p_d}[\|s_\theta(\mathbf{X})\|^2] < \infty$
    \item For any $v \sim p_v$, $\mathbb{E}_{p_d}[\|\nabla v(\mathbf{X})\|^2] < \infty$ and $\nabla v(\mathbf{X})\neq \mathbf{0}$ almost surely.
\end{enumerate}
\end{assumption}

\begin{assumption}[Boundary Conditions]
\label{assumption:boundary}
For any $v \sim p_v$:
\begin{equation}
    p_d(\mathbf{x})(\nabla v(\mathbf{x})^\top s_\theta(\mathbf{x}))\nabla v(\mathbf{x}) \to 0 \text{ as } \|\mathbf{x}\| \to \infty
\end{equation}

\end{assumption}

\subsection{Score Matching in Transformed Space}

Consider a random vector \( \mathbf{X} \in \mathbb{R}^n \) with density \( p_d(\mathbf{x}) \). Let \( v: \mathbb{R}^n \to \mathbb{R} \) be a random twice continuously differentiable function sampled from \( p_v(v) \), with \( \nabla v(\mathbf{x}) \neq \mathbf{0} \) almost surely. We denote \( q_d^v \) as the distribution of \( Y = v(\mathbf{X}) \) when \( v \) is fixed. Following the original score matching framework, we minimize:

\begin{equation}
\label{eq:appendix_sm_loss_y}
\mathcal{L}_{\text{SM}}^{(y)}(\bar{s}_\theta) = \frac{1}{2} \mathbb{E}_{p_v} \left[ \mathbb{E}_{q_d^v} \left[ \left( \bar{s}_\theta(Y) - \nabla_y \log q_d^v(Y) \right)^2 \right] \right].
\end{equation}

From the change of variables formula (Corollary~\ref{cor:score_change_variables_different_dimensions}):
\begin{equation}
\label{eq:appendix_score_change_variables}
\nabla_y \log q_d^v(y) = \frac{1}{\| \nabla v(\mathbf{x}) \|^2} \left( \nabla v(\mathbf{x})^\top \nabla_{\mathbf{x}} \log p_d(\mathbf{x}) - R(\mathbf{x}) \right)
\end{equation}

Similarly for the model score:
\begin{equation}
\label{eq:appendix_s_theta_relation}
\bar{s}_\theta(y) = \frac{1}{\| \nabla v(\mathbf{x}) \|^2} \left( \nabla v(\mathbf{x})^\top s_\theta(x) - R(\mathbf{x}) \right)
\end{equation}

where $R(\mathbf{x})= \Delta v(\mathbf{x}) + \sum_{i=1}^n \left( \frac{\partial v}{\partial x_i} \right)^2 \frac{\partial}{\partial y} \log p_d((x_j)_{j \neq i} \mid y) $.

Substituting these relations into \eqref{eq:appendix_sm_loss_y}:
\begin{equation}
\label{eq:appendix_L_intermediate}
\mathcal{L}_{\text{intermediate}}(s_\theta) = \frac{1}{2} \mathbb{E}_{p_v}\mathbb{E}_{p_d} \left[ \frac{1}{\| \nabla v(\mathbf{X}) \|^4} \left( \nabla v(\mathbf{X})^\top s_\theta(\mathbf{X}) -  \nabla v(\mathbf{X})^\top \nabla_{\mathbf{x}} \log p_d(\mathbf{X})  \right)^2  \right]
\end{equation}

\subsection{Normalization Considerations}

The $\|\nabla v(\mathbf{x})\|^4$ term in the denominator of \eqref{eq:appendix_L_intermediate} disproportionately amplifies the loss when $\|\nabla v(\mathbf{x})\|$ is small. We address this through two alternative approaches.

First, we can weight \eqref{eq:appendix_L_intermediate} by $\|\nabla v(\mathbf{x})\|^2$, yielding:
\begin{equation}
\label{eq:appendix_L_normalized}
\mathcal{L}_{\text{normalized}}(s_\theta) = \frac{1}{2} \mathbb{E}_{p_v}\mathbb{E}_{p_d} \left[  \left( \nabla \tilde{v}(\mathbf{X})^\top s_\theta(\mathbf{X}) - \nabla \tilde{v}(\mathbf{X})^\top \nabla_{\mathbf{x}} \log p_d(\mathbf{X})  \right)^2  \right]
\end{equation}
where $\nabla \tilde{v}(\mathbf{x}) = \frac{\nabla v(\mathbf{x})}{\|\nabla v(\mathbf{x})\|}$.
This normalized loss is gradient magnitude-invariant during optimization. However, the normalization introduces computational overhead and makes the variance reduction version of the loss intractable.

As an alternative, we weight \eqref{eq:appendix_L_intermediate} by $\|\nabla v(\mathbf{x})\|^4$, resulting in:
\begin{equation}
\label{eq:appendix_L_unnormalized}
\mathcal{L}_{\text{un-normalized}}(s_\theta) = \frac{1}{2} \mathbb{E}_{p_v}\mathbb{E}_{p_d} \left[  \left( \nabla {v}(\mathbf{X})^\top s_\theta(\mathbf{X}) - \nabla {v}(\mathbf{X})^\top \nabla_{\mathbf{x}} \log p_d(\mathbf{X})  \right)^2  \right]
\end{equation}
While this loss places greater emphasis on regions where $\|\nabla v(\mathbf{x})\|$ is large, it proves more tractable in practice. We proceed with loss \eqref{eq:appendix_L_unnormalized} in the following derivations, noting that substituting $\tilde{v}(\mathbf{x})$ for $v(\mathbf{x})$ preserves all subsequent mathematical relationships.

\subsection[Eliminating the Dependence on Analytical Score]{Eliminating the Dependence on $\nabla_{\mathbf{x}} \log p(\mathbf{x})$}

To eliminate $\nabla_{\mathbf{x}} \log p_d(\mathbf{x})$, we expand the square and apply integration by parts.

First, expand the squared term:
\begin{align*}
  &\frac{1}{2} \mathbb{E}_{p_v}\mathbb{E}_{p_d} \left[  \left( \nabla {v}(\mathbf{x})^\top s_\theta(\mathbf{x}) - \nabla {v}(\mathbf{x})^\top \nabla_{\mathbf{x}} \log p_d(\mathbf{x})  \right)^2  \right]\\
  &=  \frac{1}{2} \mathbb{E}_{p_v}\mathbb{E}_{p_d} \left[\left( \nabla v(\mathbf{x})^\top s_\theta(\mathbf{x})\right)^2\right]
     - \mathbb{E}_{p_v}\mathbb{E}_{p_d} \left[ \left( \nabla v(\mathbf{x})^\top s_\theta(\mathbf{x}) \right) \left( \nabla v(\mathbf{x})^\top \nabla_{\mathbf{x}} \log p_d(\mathbf{x}) \right)\right] +C 
\end{align*}

We simplify the second part of the loss using integration by parts.

Consider the term:
\begin{align}
&\mathbb{E}_{p_v}\mathbb{E}_{p_d} \left[ \left( \nabla v(\mathbf{X})^\top s_\theta(\mathbf{X}) \right) \left( \nabla v(\mathbf{X})^\top \nabla_{\mathbf{x}} \log p_d(\mathbf{X}) \right)\right]\\
 &= \mathbb{E}_{p_v}\mathbb{E}_{p_d} \left[ \left( \nabla v(\mathbf{X})^\top s_\theta(\mathbf{X}) \right) \left( \nabla v(\mathbf{X})^\top \frac{\nabla_{\mathbf{x}} p_d(\mathbf{X})}{p_d(\mathbf{X})} \right) \right] \\
&= \mathbb{E}_{p_v}\left[\int \left( \nabla v(\mathbf{x})^\top s_\theta (\mathbf{x})\right) \left( \nabla v(\mathbf{x})^\top \frac{\nabla_{\mathbf{x}} p_d(\mathbf{x})}{p_d(\mathbf{x})} \right) p_d(\mathbf{x}) d\mathbf{x}\right] \\
&= \mathbb{E}_{p_v}\left[\int \left( \nabla v(\mathbf{x})^\top s_\theta (\mathbf{x})\right) \left( \nabla v(\mathbf{x})^\top \nabla_{\mathbf{x}} p_d(\mathbf{x}) \right) d\mathbf{x}\right].
\end{align}

We integrate by parts taking in account assumption \ref{assumption:boundary}:
\begin{align}
 &\mathbb{E}_{p_v}\left[\int \left( \nabla v(\mathbf{x})^\top s_\theta \right) \left( \nabla v(\mathbf{x})^\top \nabla_{\mathbf{x}} p_d(\mathbf{x}) \right) d\mathbf{x}\right] \\
 &=  \mathbb{E}_{p_v}\left[- \int p_d(\mathbf{x}) \nabla_{\mathbf{x}} \cdot \left( \left( \nabla v(\mathbf{x})^\top s_\theta(\mathbf{x}) \right) \nabla v(\mathbf{x}) \right) d\mathbf{x}\right].
\end{align}

Compute the divergence:
\begin{align}
\nabla_{\mathbf{x}} \cdot \left( \left( \nabla v^\top(\mathbf{x}) s_\theta(\mathbf{x}) \right) \nabla v(\mathbf{x}) \right) &= \left( \nabla v(\mathbf{x})^\top s_\theta(\mathbf{x}) \right) \Delta v(\mathbf{x})\\
& + \left( s_\theta(\mathbf{x})^\top \mathbf{H}_v(\mathbf{x}) \right) \nabla v(\mathbf{x}) + \nabla v(\mathbf{x})^\top \nabla_{\mathbf{x}} s_\theta(\mathbf{x}) \nabla v(\mathbf{x}).
\end{align}

Thus,
\begin{align*}
&\mathbb{E}_{p_v}\mathbb{E}_{p_d} \left[ \left( \nabla v(\mathbf{X})^\top s_\theta(\mathbf{X}) \right) \left( \nabla v(\mathbf{X})^\top \nabla_{\mathbf{x}} \log p_d(\mathbf{X}) \right) \right] = \\
&- \mathbb{E}_{p_v}\mathbb{E}_{p_d} \left[ \left( \nabla v(\mathbf{X})^\top s_\theta(\mathbf{X}) \right) \Delta v(\mathbf{X}) + s_\theta(\mathbf{X})^\top \mathbf{H}_v(\mathbf{X}) \nabla v(\mathbf{X}) + \nabla v(\mathbf{X})^\top \nabla_{\mathbf{X}} s_\theta \nabla v(\mathbf{X}) \right].
\end{align*}

Collecting terms, the loss becomes:
\begin{align}
\mathcal{L}_{\text{GSSM}}(s_\theta) &= \frac{1}{2} \mathbb{E}_{p_v}\mathbb{E}_{p_d} \left[ \left( \nabla v(\mathbf{X})^\top s_\theta(\mathbf{X}) \right)^2 \right] + \mathbb{E}_{p_v}\mathbb{E}_{p_d} \left[ \nabla v(\mathbf{X})^\top \nabla_{\mathbf{x}} s_\theta(\mathbf{X}) \nabla v(\mathbf{X}) \right] \nonumber \\
&+ \mathbb{E}_{p_v}\mathbb{E}_{p_d} \left[ s_\theta(\mathbf{X})^\top \mathbf{H}_v(\mathbf{X}) \nabla v(\mathbf{X}) \right] + \mathbb{E}_{p_v}\mathbb{E}_{p_d} \left[ \nabla v(\mathbf{X})^\top s_\theta(\mathbf{X})  \Delta v(\mathbf{X})  \right].\label{eq:appendix_gssm_final}
\end{align}

\section{Quadratic GSSM Details}
\label{appendix:quadratic_gssm}
We choose a quadratic random function:
\begin{equation}
v(\mathbf{x}) = \frac{1}{2}\mathbf{x}^\top \mathbf{A} \mathbf{x} + \mathbf{b}^\top \mathbf{x}
\end{equation}
where $\mathbf{A} \in \mathbb{R}^{n \times n}$ is a random symmetric matrix, and $\mathbf{b} \in \mathbb{R}^n$ is a random vector. The random variables are constructed with the following properties:
\begin{itemize}
   \item \textbf{Matrix $\mathbf{A}$}:
   \begin{itemize}
       \item Symmetric: $\mathbf{A} = \mathbf{A}^\top$
       \item Diagonal entries: $A_{ii} \sim \mathcal{N}(0, \sigma_1^2)$
       \item Off-diagonal entries: $A_{ij} = A_{ji} \sim \mathcal{N}(0, \sigma_2^2)$ for $i < j$
       \item Independence of entries, except that off-diagonal entries $A_{ij}$ and $A_{ji}$ are dependent due to symmetry
   \end{itemize}
   \item \textbf{Vector $\mathbf{b}$}:
   \begin{itemize}
       \item  $\mathbf{b}\sim \mathcal{N}(0, \sigma_3^2I)$
       \item Independent from $\mathbf{A}$
   \end{itemize}
\end{itemize}

\subsection{Quadratic GSSM Loss}
The components needed for the GSSM loss are as follows:
\begin{align}
\nabla v(\mathbf{x}) &= \mathbf{A}\mathbf{x} + \mathbf{b} \label{eq:grad_v_general} \\
\mathbf{H}_v(\mathbf{x}) &= \mathbf{A} \label{eq:hess_v_general} \\
\Delta v(\mathbf{x}) &= \operatorname{tr}(\mathbf{A}) \label{eq:laplace_v_general}
\end{align}

To satisfy the conditions for the GSSM loss \ref{eq:gssm_loss}, we require $\nabla v(\mathbf{x}) \neq \mathbf{0}$. Thus we require $x \neq \mathbf{0}$.
For $\mathbf{x} \neq \mathbf{0}$, $\nabla v(\mathbf{x})$ is almost surely non-zero.

Substituting these into equation~\eqref{eq:gssm_loss}, we obtain:
\begin{align}
\mathcal{L}_{\text{GSSM}}(s_\theta) &= \frac{1}{2} \mathbb{E}_{\mathbf{A}, \mathbf{b}}\mathbb{E}_{p_d} \left[ \left( (\mathbf{A}\mathbf{X} + \mathbf{b})^\top s_\theta(\mathbf{X}) \right)^2 \right] \nonumber \\
&+ \mathbb{E}_{\mathbf{A}, \mathbf{b}}\mathbb{E}_{p_d} \left[ (\mathbf{A}\mathbf{X} + \mathbf{b})^\top \nabla_{\mathbf{x}} s_\theta(\mathbf{X}) (\mathbf{A}\mathbf{X} + \mathbf{b}) \right] \nonumber \\
&+ \mathbb{E}_{\mathbf{A}, \mathbf{b}}\mathbb{E}_{p_d} \left[ s_\theta(\mathbf{X})^\top \mathbf{A} (\mathbf{A}\mathbf{X} + \mathbf{b}) \right] \nonumber \\
&+ \mathbb{E}_{\mathbf{A}, \mathbf{b}}\mathbb{E}_{p_d} \left[ (\mathbf{A}\mathbf{X} + \mathbf{b})^\top s_\theta(\mathbf{X}) \operatorname{tr}(\mathbf{A}) \right] \label{eq:gssm_terms}
\end{align}

\subsection{Variance Reduction for Quadratic GSSM}
To reduce the variance of our estimator, we integrate out the randomness in $\mathbf{A}$ and $\mathbf{b}$. We compute the expectations of each term in the GSSM loss with respect to the randomness in $\mathbf{A}$ and $\mathbf{b}$, assuming $s_\theta(\mathbf{x})$ is independent of $\mathbf{A}$ and $\mathbf{b}$.

\subsubsection{Expectation of the Quadratic Term}
We first compute the expectation of the quadratic term in Equation~\eqref{eq:gssm_terms}:

\begin{align}
\mathbb{E}_{\mathbf{A}, \mathbf{b}} \left[ \left( (\mathbf{A}\mathbf{x} + \mathbf{b})^\top s_\theta(\mathbf{x}) \right)^2 \right] &= \mathbb{E}_{\mathbf{A}} \left[ \left( s_\theta(\mathbf{x})^\top \mathbf{A} \mathbf{x} \right)^2 \right]\\
&+ 2 \mathbb{E}_{\mathbf{A}, \mathbf{b}} \left[ \left( s_\theta(\mathbf{x})^\top \mathbf{A} \mathbf{x} \right) \left( s_\theta(\mathbf{x})^\top \mathbf{b} \right) \right] + \mathbb{E}_{\mathbf{b}} \left[ \left( s_\theta(\mathbf{x})^\top \mathbf{b} \right)^2 \right]
\end{align}

Since $\mathbf{A}$ and $\mathbf{b}$ have zero mean and are independent, the cross-term vanishes:
\begin{equation}
\mathbb{E}_{\mathbf{A}, \mathbf{b}} \left[ \left( s_\theta(\mathbf{x})^\top \mathbf{A} \mathbf{x} \right) \left( s_\theta(\mathbf{x})^\top \mathbf{b} \right) \right] = 0
\end{equation}

Thus, we have:
\begin{equation}
\mathbb{E}_{\mathbf{A}, \mathbf{b}} \left[ \left( (\mathbf{A}\mathbf{x} + \mathbf{b})^\top s_\theta(\mathbf{x}) \right)^2 \right] = \mathbb{E}_{\mathbf{A}} \left[ \left( s_\theta(\mathbf{x})^\top \mathbf{A} \mathbf{x} \right)^2 \right] + \mathbb{E}_{\mathbf{b}} \left[ \left( s_\theta(\mathbf{x})^\top \mathbf{b} \right)^2 \right]
\end{equation}

\paragraph{Computing $\mathbb{E}_{\mathbf{A}} \left[ \left( s_\theta(\mathbf{x})^\top \mathbf{A} \mathbf{x} \right)^2 \right]$}
We expand and compute:
\begin{align}
\mathbb{E}_{\mathbf{A}} \left[ \left( s_\theta(\mathbf{x})^\top \mathbf{A} \mathbf{x} \right)^2 \right] &= \mathbb{E}_{\mathbf{A}} \left[ \left( \sum_{i,j} s_{\theta i}(\mathbf{x}) A_{ij} x_j \right)^2 \right] \\
&= \sum_{i,j,k,l} s_{\theta i}(\mathbf{x}) s_{\theta k}(\mathbf{x}) x_j x_l \mathbb{E}[A_{ij} A_{kl}]
\end{align}

Using the properties of $\mathbf{A}$, the expectation $\mathbb{E}[A_{ij} A_{kl}]$ is non-zero only when $(i,j) = (k,l)$ or $(i,j) = (l,k)$. Therefore:
\begin{equation}
\mathbb{E}[A_{ij} A_{kl}] = \begin{cases}
\sigma_1^2 & \text{if } i = j = k = l \\
\sigma_2^2 & \text{if } i = k, j = l, i \neq j \\
\sigma_2^2 & \text{if } i = l, j = k, i \neq j \\
0 & \text{otherwise}
\end{cases}
\end{equation}

The expected value becomes:
\begin{align}
\mathbb{E}_{\mathbf{A}} \left[ \left( s_\theta(\mathbf{x})^\top \mathbf{A} \mathbf{x} \right)^2 \right] &=  \left(\sigma_1^2-2\sigma_2^2\right) \sum_{i} s_{\theta i}(\mathbf{x})^2 x_i^2 + \sigma_2^2 \left( \sum_{i \neq j} s_{\theta i}(\mathbf{x})^2 x_j^2 +  \sum_{i<j} s_{\theta i}(\mathbf{x}) s_{\theta j}(\mathbf{x}) x_i x_j \right) \nonumber \\
&=\left( \sigma_1^2 - 2\sigma_2^2 \right) \sum_{i} s_{\theta i}(\mathbf{x})^2 x_i^2 + \sigma_2^2 \left( \| \mathbf{x} \|^2 \| s_\theta(\mathbf{x}) \|^2 + \left( s_\theta(\mathbf{x})^\top \mathbf{x} \right)^2 \right)
\end{align}

Since the entries of $\mathbf{b}$ are independent with variance $\sigma_3^2$, we have:
\begin{equation}
\mathbb{E}_{\mathbf{b}} \left[ \left( s_\theta(\mathbf{x})^\top \mathbf{b} \right)^2 \right] = \sigma_3^2 \| s_\theta(\mathbf{x}) \|^2
\end{equation}

Combining both results, we obtain:
\begin{align}
\mathbb{E}_{\mathbf{A}, \mathbf{b}} \left[ \left( (\mathbf{A}\mathbf{x} + \mathbf{b})^\top s_\theta(\mathbf{x}) \right)^2 \right] &= \left( \sigma_1^2 - 2\sigma_2^2 \right) \sum_{i} s_{\theta i}(\mathbf{x})^2 x_i^2 \nonumber \\
&+ \sigma_2^2 \left( \| \mathbf{x} \|^2 \| s_\theta(\mathbf{x}) \|^2 + \left( s_\theta(\mathbf{x})^\top \mathbf{x} \right)^2 \right) + \sigma_3^2 \| s_\theta(\mathbf{x}) \|^2 \label{eq:expectation_T1}
\end{align}

\subsubsection{Expectation of the Hessian Term}
The Hessian term from Equation~\eqref{eq:gssm_terms} is:
\begin{equation}
\mathbb{E}_{\mathbf{A}, \mathbf{b}} \left[ s_\theta(\mathbf{x})^\top \mathbf{A} (\mathbf{A} \mathbf{x} + \mathbf{b}) \right] = \mathbb{E}_{\mathbf{A}} \left[ s_\theta(\mathbf{x})^\top \mathbf{A}^2 \mathbf{x} \right] + \mathbb{E}_{\mathbf{A}} \left[ s_\theta(\mathbf{x})^\top \mathbf{A} \right] \mathbb{E}_{\mathbf{b}} \left[ \mathbf{b} \right]
\end{equation}

Since $\mathbb{E}_{\mathbf{A}} [ \mathbf{A} ] = \mathbf{0}$ and $\mathbb{E}_{\mathbf{b}} [ \mathbf{b} ] = \mathbf{0}$, the second term vanishes. 

Computing $\mathbb{E}_{\mathbf{A}} \left[ \mathbf{A}^2 \right]$:
\begin{equation}
\mathbb{E}_{\mathbf{A}} \left[ \mathbf{A}^2 \right] = \left( \sigma_1^2 - \sigma_2^2 \right) \mathbf{I} + n \sigma_2^2 \mathbf{I} = \left( \sigma_1^2 + (n - 1) \sigma_2^2 \right) \mathbf{I}
\end{equation}

Therefore:
\begin{equation}
\mathbb{E}_{\mathbf{A}, \mathbf{b}} \left[ s_\theta(\mathbf{x})^\top \mathbf{A} (\mathbf{A} \mathbf{x} + \mathbf{b}) \right] = \left( \sigma_1^2 + (n - 1) \sigma_2^2 \right) s_\theta(\mathbf{x})^\top \mathbf{x} \label{eq:expectation_T3}
\end{equation}

\subsubsection{Expectation of the Trace Term}
The trace term from Equation~\eqref{eq:gssm_terms} is:
\begin{equation}
\mathbb{E}_{\mathbf{A}, \mathbf{b}} \left[ (\mathbf{A}\mathbf{x} + \mathbf{b})^\top s_\theta(\mathbf{x}) \operatorname{tr}(\mathbf{A}) \right] = \sigma_1^2s_\theta(\mathbf{x})^\top \mathbf{x}\label{eq:expectation_T4}
\end{equation}

\subsubsection{Final Variance-Reduced Loss}
Combining Equations~\eqref{eq:expectation_T1}, \eqref{eq:expectation_T3}, \eqref{eq:expectation_T4}, and \eqref{eq:gssm_terms}, we obtain the variance-reduced GSSM loss:
\begin{align}
\mathcal{L}_{\text{GSSM-VR}}(s_\theta) &= \frac{1}{2} \mathbb{E}_{p_d} \big[ \left( \sigma_1^2 - 2\sigma_2^2 \right) \sum_{i} s_{\theta i}(\mathbf{X})^2 X_i^2\\
&+ \sigma_2^2 \left( \| \mathbf{X} \|^2 \| s_\theta(\mathbf{X}) \|^2 + \left( s_\theta(\mathbf{X})^\top \mathbf{X} \right)^2 \right) + \sigma_3^2 \| s_\theta(\mathbf{X}) \|^2 \big] \nonumber \\
&+ \mathbb{E}_{p_d} \left[ \left( 2\sigma_1^2 + (n - 1) \sigma_2^2 \right) s_\theta(\mathbf{X})^\top \mathbf{X} \right] \nonumber \\
&+ \mathbb{E}_{p_d}\mathbb{E}_{\mathbf{A}, \mathbf{b}} \left[ (\mathbf{A}\mathbf{X} + \mathbf{b})^\top \nabla_{\mathbf{x}} s_\theta(\mathbf{X}) (\mathbf{A}\mathbf{X} + \mathbf{b}) \right] \label{eq:final_gssm_vr}
\end{align}

\subsection[Efficient Sampling]{Sampling from \(\mathbf{A}\mathbf{x}\)}

In this section, we detail the construction of the random symmetric matrix \(\mathbf{A}\) used in our quadratic transformation and describe how to efficiently sample \(\mathbf{A}\mathbf{x}\) without explicitly forming \(\mathbf{A}\).

\paragraph{Gaussian Orthogonal Ensemble (GOE) Matrices}
We define \(\mathbf{A} \in \mathbb{R}^{n \times n}\) as a random symmetric matrix drawn from the Gaussian Orthogonal Ensemble (GOE) with entries:
\begin{align*}
A_{ii} &\sim \mathcal{N}(0, \sigma^2), \quad \text{for } i = 1, \dots, n, \\
A_{ij} &= A_{ji} \sim \mathcal{N}(0, \sigma^2/2), \quad \text{for } i < j.
\end{align*}
This ensures that each off-diagonal element has variance \(\sigma^2/2\), while diagonal elements have variance \(\sigma^2\).

\paragraph{Distribution of \(\mathbf{A}\mathbf{x}\)}
For a fixed vector \(\mathbf{x} \in \mathbb{R}^n\), the product \(\mathbf{A}\mathbf{x}\) is a zero-mean Gaussian vector:
\begin{equation*}
\mathbf{A}\mathbf{x} \sim \mathcal{N}\left( \mathbf{0},\, \boldsymbol{\Sigma} \right),
\end{equation*}
where the covariance matrix \(\boldsymbol{\Sigma}\) is given by:
\begin{equation*}
\boldsymbol{\Sigma} = \mathbb{E}\left[ \mathbf{A}\mathbf{x} (\mathbf{A}\mathbf{x})^\top \right] = \frac{\sigma^2}{2} \left( \|\mathbf{x}\|^2 \mathbf{I}_n + \mathbf{x}\mathbf{x}^\top \right).
\end{equation*}

\paragraph{Efficient Sampling Procedure}
To sample \(\mathbf{A}\mathbf{x}\) without constructing \(\mathbf{A}\), we use the following steps:
\begin{enumerate}
    \item Sample \(\boldsymbol{\epsilon} \sim \mathcal{N}(\mathbf{0}, \mathbf{I}_n)\).
    \item Sample \(z \sim \mathcal{N}(0, 1)\), independently of \(\boldsymbol{\epsilon}\).
    \item Compute \(\mathbf{A}\mathbf{x}\) as:
    \begin{equation*}
    \mathbf{A}\mathbf{x} = \frac{\sigma}{\sqrt{2}} \left( \|\mathbf{x}\| \boldsymbol{\epsilon} + z \mathbf{x} \right).
    \end{equation*}
\end{enumerate}

\subsection{Additional Experimental Results}

Beyond the score matching loss presented in the main text, we also evaluated the models using test log-likelihood. These estimates were computed using Annealed Importance Sampling (AIS) with a proposal distribution of $\mathcal{N}(0, 2\mathbf{I})$ and $1,000,000$ samples, following the methodology of \cite{song2019slicedscorematchingscalable}. Figure~\ref{fig:test_ll} presents these results across all three UCI datasets.

\begin{figure}[ht]
    \begin{center}
    \includegraphics[width=0.9\textwidth]{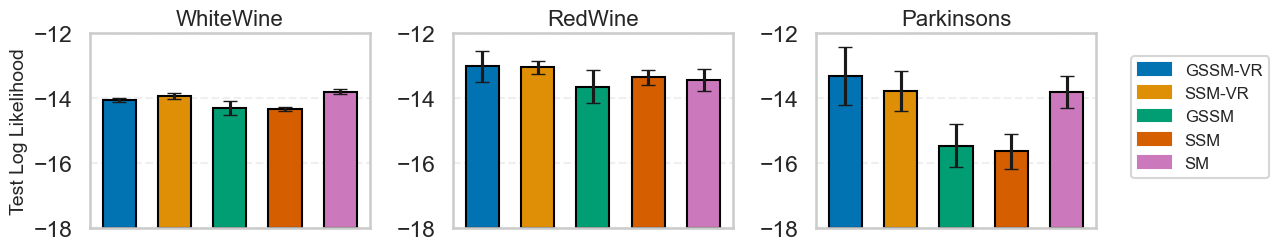}
    \caption{Test log-likelihood for DKEF models trained with different loss functions. Higher values indicate better performance.}
    \label{fig:test_ll}
    \end{center}
\end{figure}

The test log-likelihood results largely align with our score matching loss findings. GSSM-VR demonstrates superior performance on RedWine and Parkinsons datasets, while showing slightly lower performance on WhiteWine.

\section{Deep Kernel Exponential Families with Generalized Sliced Score Matching}
\label{appendix:gssm_dkef}

For our experiments, we follow the setup of \cite{song2019slicedscorematchingscalable}, who demonstrated the effectiveness of sliced score matching with Deep Kernel Exponential Families (DKEFs). We maintain their model architecture and training procedure, changing only the loss function to our GSSM objective.

\subsection{Deep Kernel Exponential Family Models}

Following \cite{wenliang2021learningdeepkernelsexponential} and \cite{song2019slicedscorematchingscalable}, we approximate the unnormalized log density as:
\begin{equation}
\label{eq:dkef_density}
\log \tilde{p}_f(\mathbf{x}) = f(\mathbf{x}) + \log q_0(\mathbf{x}),
\end{equation}
where $q_0(\mathbf{x})$ is a fixed function, and we represent $f$ as with Nyström approximation with $L$ inducing points $\{\mathbf{z}_l\}_{l=1}^L$:
\begin{equation}
\label{eq:f_nystrom}
f(\mathbf{x}) = \sum_{l=1}^L \alpha_l k(\mathbf{x}, \mathbf{z}_l),
\end{equation}
where $\boldsymbol{\alpha} = (\alpha_1, \dots, \alpha_L)^\top$ are the model parameters to be learned.

Following \cite{song2019slicedscorematchingscalable}, we use their deep kernel architecture, which expresses the kernel as a mixture of 3 Guassian kernels:
\begin{equation}
k_w(\mathbf{x}, \mathbf{y}) = \sum_{r=1}^{3} \rho_r \exp \left( -\frac{1}{2\sigma_r^2} \|\phi_{w_r}(\mathbf{x}) - \phi_{w_r}(\mathbf{y})\|^2 \right).
\end{equation}
In the above equation:
\begin{itemize}
    \item $\phi_{w_r}(\cdot)$ denotes a feature mapping associated with the $r$th component of the kernel, parameterized by $w_r$.
    \item $\sigma_r$ represents the length scale for the $r$th component.
    \item $\rho_r$ are nonnegative mixture coefficients that weight the contribution of each component to the overall kernel.
\end{itemize}

\subsection{Quadratic Form of the GSSM Loss}

As in standard score matching or sliced score matching, the GSSM loss also reduces to a quadratic form in terms of the model parameters $\boldsymbol{\alpha}$.\subsection{Quadratic Form of GSSM Loss}
We now show that our GSSM loss, like SSM, becomes quadratic in $\boldsymbol{\alpha}$. For DKEF models, the score function is linear in $\boldsymbol{\alpha}$:
\begin{equation}
s_{\boldsymbol{\theta}}(\mathbf{x}) = \sum_{l=1}^L \alpha_l \nabla_{\mathbf{x}} k(\mathbf{x}, \mathbf{z}_l) + \nabla_{\mathbf{x}} \log q_0(\mathbf{x})
\end{equation}

From equation \eqref{eq:gssm_loss}, our GSSM loss is:
\begin{align}
\mathcal{L}_{\text{GSSM}}(s_\theta) &= \frac{1}{2}\mathbb{E}_{p_d} \mathbb{E}_{p_v} \left[ \left( \nabla_{\mathbf{x}}v^\top(\mathbf{X}) s_\theta(\mathbf{X}) \right)^2 \right] + \mathbb{E}_{p_d} \mathbb{E}_{p_v}  \left[\nabla_{\mathbf{x}} v^\top(\mathbf{X}) \nabla_{\mathbf{x}} s_\theta(\mathbf{X})\nabla_{\mathbf{x}} v(\mathbf{X}) \right] \nonumber \\
&\quad + \mathbb{E}_{p_d} \mathbb{E}_{p_v}  \left[ s_\theta(\mathbf{X})^\top \mathbf{H}_v(\mathbf{X})\nabla_{\mathbf{x}}v(\mathbf{X}) \right] + \mathbb{E}_{p_d} \mathbb{E}_{p_v}  \left[ \nabla_{\mathbf{x}}v^\top(\mathbf{X}) s_\theta(\mathbf{X})  \Delta v(\mathbf{X})  \right].
\end{align}

Let us examine each term:

\subsubsection{First Term}
Define $g_l(\mathbf{x}) = \nabla_{\mathbf{x}} k(\mathbf{x}, \mathbf{z}_l)$ and $h(\mathbf{x}) = \nabla_{\mathbf{x}} \log q_0(\mathbf{x})$. Then:
\begin{align}
\left( \nabla v(\mathbf{x})^\top s_\theta(\mathbf{x}) \right)^2 &= \left(\nabla v(\mathbf{x})^\top \left(\sum_{l=1}^L \alpha_l g_l(\mathbf{x}) + h(\mathbf{x})\right)\right)^2 \nonumber \\
&= \left(\sum_{l=1}^L \alpha_l \nabla v(\mathbf{x})^\top g_l(\mathbf{x}) + \nabla v(\mathbf{x})^\top h(\mathbf{x})\right)^2
\end{align}

Define $\tilde{g}_l(\mathbf{x}) = \nabla v(\mathbf{x})^\top g_l(\mathbf{x})$ and $\tilde{h}(\mathbf{x}) = \nabla v(\mathbf{x})^\top h(\mathbf{x})$. Then:
\begin{align}
\frac{1}{2}\mathbb{E}_{p_v}\mathbb{E}_{p_d}\left[\left( \nabla v(\mathbf{X})^\top s_\theta(\mathbf{X}) \right)^2\right] &= \frac{1}{2} \boldsymbol{\alpha}^\top \mathbf{G}_1 \boldsymbol{\alpha} + \boldsymbol{\alpha}^\top \mathbf{b}_1 + C_1
\end{align}
where
\begin{align}
(\mathbf{G}_1)_{l,l'} &= \mathbb{E}_{p_v}\mathbb{E}_{p_d} \left[\tilde{g}_l(\mathbf{X})\tilde{g}_{l'}(\mathbf{X})\right] \\
(\mathbf{b}_1)_l &= \mathbb{E}_{p_v}\mathbb{E}_{p_d} \left[\tilde{g}_l(\mathbf{X})\tilde{h}(\mathbf{X})\right]
\end{align}

\subsubsection{Second Term}
We see:
\begin{align}
\nabla v(\mathbf{x})^\top \nabla_{\mathbf{x}} s_\theta(\mathbf{x}) \nabla v(\mathbf{x}) &= \nabla v(\mathbf{x})^\top \left(\sum_{l=1}^L \alpha_l \nabla_{\mathbf{x}} \nabla_{\mathbf{x}}^\top k(\mathbf{x}, \mathbf{z}_l) + \nabla_{\mathbf{x}} \nabla_{\mathbf{x}}^\top \log q_0(\mathbf{x})\right)\nabla v(\mathbf{x}) \nonumber \\
&= \sum_{l=1}^L \alpha_l \nabla v(\mathbf{x})^\top \nabla_{\mathbf{x}} \nabla_{\mathbf{x}}^\top k(\mathbf{x}, \mathbf{z}_l) \nabla v(\mathbf{x}) + \nabla v(\mathbf{x})^\top \nabla_{\mathbf{x}} \nabla_{\mathbf{x}}^\top \log q_0(\mathbf{x})\nabla v(\mathbf{x})
\end{align}

Therefore:
\begin{align}
\mathbb{E}_{p_v}\mathbb{E}_{p_d}\left[\nabla v(\mathbf{X})^\top \nabla_{\mathbf{x}} s_\theta(\mathbf{X}) \nabla v(\mathbf{X})\right] &= \boldsymbol{\alpha}^\top \mathbf{b}_2 + C_2
\end{align}
where
\begin{align}
(\mathbf{b}_2)_l &= \mathbb{E}_{p_v}\mathbb{E}_{p_d} \left[\nabla v(\mathbf{X})^\top \nabla_{\mathbf{x}} \nabla_{\mathbf{x}}^\top k(\mathbf{X}, \mathbf{z}_l) \nabla v(\mathbf{X})\right]
\end{align}

\subsubsection{Third Term}
For the Hessian term:
\begin{align}
s_\theta(\mathbf{x})^\top \mathbf{H}_v(\mathbf{x}) \nabla v(\mathbf{x}) &= \left(\sum_{l=1}^L \alpha_l g_l(\mathbf{x}) + h(\mathbf{x})\right)^\top \mathbf{H}_v(\mathbf{x}) \nabla v(\mathbf{x}) \nonumber \\
&= \sum_{l=1}^L \alpha_l g_l(\mathbf{x})^\top \mathbf{H}_v(\mathbf{x}) \nabla v(\mathbf{x}) + h(\mathbf{x})^\top \mathbf{H}_v(\mathbf{x}) \nabla v(\mathbf{x})
\end{align}

Therefore:
\begin{align}
\mathbb{E}_{p_v}\mathbb{E}_{p_d}\left[s_\theta(\mathbf{X})^\top \mathbf{H}_v(\mathbf{X}) \nabla v(\mathbf{X})\right] &= \boldsymbol{\alpha}^\top \mathbf{b}_3 + C_3
\end{align}
where
\begin{align}
(\mathbf{b}_3)_l &= \mathbb{E}_{p_v}\mathbb{E}_{p_d} \left[g_l(\mathbf{X})^\top \mathbf{H}_v(\mathbf{X}) \nabla v(\mathbf{X})\right]
\end{align}

\subsubsection{Fourth Term}
For the Laplacian term:
\begin{align}
\nabla v(\mathbf{x})^\top s_\theta(\mathbf{x})  \Delta v(\mathbf{x}) &= \Delta v(\mathbf{x}) \left(\sum_{l=1}^L \alpha_l \nabla v(\mathbf{x})^\top g_l(\mathbf{x}) + \nabla v(\mathbf{x})^\top h(\mathbf{x})\right).
\end{align}
Therefore:
\begin{align}
\mathbb{E}_{p_v}\mathbb{E}_{p_d}\left[\nabla v(\mathbf{X})^\top s_\theta(\mathbf{X})  \Delta v(\mathbf{X})\right] &=\boldsymbol{\alpha}^\top \mathbf{b}_4 + C_4
\end{align}

where
\begin{align}
(\mathbf{b}_4)_l &= \mathbb{E}_{p_v}\mathbb{E}_{p_d} \left[\Delta v(\mathbf{X}) \nabla v(\mathbf{X})^\top g_l(\mathbf{X})\right]
\end{align}

\subsubsection{Complete Loss}
Combining all terms, the total loss is quadratic in $\boldsymbol{\alpha}$:
\begin{equation}
\mathcal{L}_{\text{GSSM}}(\boldsymbol{\alpha}) = \frac{1}{2} \boldsymbol{\alpha}^\top \mathbf{G}_1 \boldsymbol{\alpha} + \boldsymbol{\alpha}^\top (\mathbf{b}_1 + \mathbf{b}_2 + \mathbf{b}_3+ \mathbf{b}_4) + C
\end{equation}

As in \cite{song2019slicedscorematchingscalable}, we add $\ell_2$ regularization:
\begin{equation}
\hat{\mathcal{L}}_{\text{GSSM}}(\boldsymbol{\alpha}) = \mathcal{L}_{\text{GSSM}}(\boldsymbol{\alpha}) + \frac{1}{2} \lambda_{\boldsymbol{\alpha}} \| \boldsymbol{\alpha} \|^2
\end{equation}
yielding the solution:
\begin{equation}
\boldsymbol{\alpha}^* = -(\mathbf{G}_1 + \lambda_{\boldsymbol{\alpha}}\mathbf{I})^{-1}(\mathbf{b}_1 + \mathbf{b}_2 + \mathbf{b}_3+\mathbf{b}_4)
\end{equation}

The derivation for the variance-reduced version follows similarly.

\subsection{Implementation Details}

All hyperparameters (such as network architectures, optimization settings, and the choice of $q_0(\mathbf{x})$) follow \cite{song2019slicedscorematchingscalable} to ensure a fair comparison and reproducibility. The only difference is that we replace the original loss with our GSSM-based loss. This maintains consistency with previous experiments while highlighting the direct impact of our new score matching objective.

\section{Chess Diffusion Details}
\label{appendix:chess_derivations}

\subsection{SDE Parameters and Transformation}
Let $\mathbf{f}(\mathbf{x}, t)$ and $g(t)$ be the drift and diffusion coefficients. We use a Variance Preserving (VP) SDE with:
\begin{align}
\mathbf{f}(\mathbf{x}, t) &= -\frac{1}{2}\beta(t)\mathbf{x}, \\
g(t) &= \sqrt{\beta(t)},
\end{align}
where $\beta(t)$ is the noise schedule. The additive logistic transformation is:
\begin{align}
\phi_i(\mathbf{x}) &= \frac{e^{x_i}}{1 + \sum_{j=1}^{12} e^{x_j}}, \quad i = 1, \dots, 12,
\end{align}
which maps vectors in $\mathbb{R}^{12}$ to the probability simplex.

\subsection{Transformed Drift and Diffusion Coefficients}
The transformed reverse drift $\hat{\mathbf{f}}(\mathbf{y}, t)$ and diffusion $\tilde{\mathbf{G}}(\mathbf{y}, t)$ are:
\begin{align}
\hat{f}_i(\mathbf{y}, t) &= y_i \left[ k_i - \sum_{j=1}^{12} y_j k_j \right], \\
\tilde{G}_{ij}(\mathbf{y}, t) &= 
\begin{cases}
g(t) y_i (1 - y_i), & \text{if } i = j, \\
-g(t) y_i y_j, & \text{if } i \neq j,
\end{cases}
\end{align}
where $k_i = f_i(\phi^{-1}(\mathbf{y}), t) + \frac{1}{2} g^2(t) ( 2 y_i-1)$.

\subsection{Training Details}
\label{appendix:training_details}

We trained a score-based model $s_\theta(\mathbf{x}, t)$ using a standard U-Net architecture \cite{ronneberger2015u}. The model was trained on a dataset of chess positions obtained from Kaggle \cite{badhe_chess_evaluations_2020}. Training was conducted using an NVIDIA GeForce RTX 4090 GPU. 
\end{document}